\documentclass[dvipsnames]{article} %
\usepackage{template_conference}

\usepackage{booktabs}
\usepackage{graphicx}
\usepackage{enumitem}
\usepackage{wrapfig}
\usepackage{algorithm}
\usepackage{algorithmic}
\usepackage[utf8]{inputenc}
\DeclareUnicodeCharacter{FF5E}{\textasciitilde}

\usepackage{adjustbox}
\usepackage{amsmath}
\usepackage{colortbl}
\usepackage[utf8]{inputenc}
\definecolor{lightgray}{rgb}{0.9,0.9,0.9}
\usepackage{caption}
\usepackage{subcaption}
\usepackage[svgnames]{xcolor}
\usepackage{setspace}
\usepackage{url}
\usepackage{multirow}
\usepackage{colortbl}
\usepackage{tabularx}
\usepackage{blindtext}
\usepackage{pgfplots}
\pgfplotsset{compat=1.18} 
\usepackage{tikz}
\usetikzlibrary{er,positioning}
\usepackage{makecell}
\usepackage{tipa}
\usepackage{siunitx}
\usepackage{tocloft}
\usepackage{listings}
\usepackage{adjustbox}
\usepackage{xurl}
\usepackage{rotating}
\usepackage[normalem]{ulem}
\usepackage{multicol}
\usepackage{titlesec}
\usepackage{tcolorbox}
\tcbuselibrary{breakable,theorems}
\usepackage{transparent}
\usepackage{mathtools}

\usepackage{amsthm}

\usepackage{amsfonts}
\usepackage{etoc}
\usepackage{color}
\usepackage{autobreak}
\usepackage{soul}
\usepackage{amssymb}
\usepackage{bbding}

\newtheorem{theorem}{Theorem}[section]

\newtheorem{lemma}[theorem]{Lemma}

\theoremstyle{definition}
\newtheorem{definition}[theorem]{Definition}

\theoremstyle{remark}

\newtcolorbox{codeblock}{
  colback=gray!10,   
  colframe=gray!50,  
  boxrule=0.5mm,     
  arc=2mm,           
  left=5pt,          
  top=5pt,          
  bottom=5pt,        
}

\useunder{\uline}{\ul}{}

\usepackage{amsmath,amsfonts,bm}

\def\eqref#1{equation~\ref{#1}}

\def\1{\bm{1}}

\DeclareMathAlphabet{\mathsfit}{\encodingdefault}{\sfdefault}{m}{sl}
\SetMathAlphabet{\mathsfit}{bold}{\encodingdefault}{\sfdefault}{bx}{n}

\renewcommand{\ghlink}{YourGithubLink}

\newcommand*\justify{%
  \fontdimen2\font=0.4em
  \fontdimen3\font=0.2em
  \fontdimen4\font=0.1em
  \fontdimen7\font=0.1em
  \hyphenchar\font=`\-
}

\renewcommand{\texttt}[1]{%
  \begingroup
  \ttfamily
  \begingroup\lccode`~=`/\lowercase{\endgroup\def~}{/\discretionary{}{}{}}%
  \begingroup\lccode`~=`[\lowercase{\endgroup\def~}{[\discretionary{}{}{}}%
  \begingroup\lccode`~=`.\lowercase{\endgroup\def~}{.\discretionary{}{}{}}%
  \catcode`/=\active\catcode`[=\active\catcode`.=\active
  \justify\scantokens{#1\noexpand}%
  \endgroup
}

\definecolor{blueviolet}{RGB}{138,43,226}
\newtcolorbox{insightblock}{
  colback=blueviolet!5,   
  colframe=blueviolet!50!black!50!,    
  boxrule=0.5mm,       
  arc=2mm,            
  left=0pt,            
  right=8pt,        
  top=8pt,             
  bottom=8pt,         
}

\newtcolorbox{abstractbox}{
    colback=blue!5!white,     
    frame empty,             
    boxrule=1pt,             
    arc=4mm,                 
    left=8pt,                 
    right=8pt,                
    top=8pt,                  
    bottom=8pt,                
    opacityback=0.9
}

\newtcolorbox{examplebox}[1][]{
    breakable,
    colback=black!4,
    colframe=black!35,
    boxrule=0pt,
    sharp corners,
    boxsep=0pt,
    left=8pt,
    right=7pt,
    top=6pt,
    bottom=6pt,
    before upper={
        \if\relax\detokenize{#1}\relax
        \else
            {\small\bfseries\color{black!85} #1\par}
            \vspace{3pt}
        \fi
    },
    before skip=0.8\baselineskip,
    after skip=0.8\baselineskip,
    pad at break*=1.2mm
}

\title{CARE: Causally-Aligned Reasoning Exploration for Medical Large Language Models}

\author{
\footnotesize
\textbf{Yucheng Zhou}\textsuperscript{*,1},~
\textbf{Peng Luo}\textsuperscript{*,1},~
\textbf{Qianning Wang}\textsuperscript{2},~
\textbf{Chengzhong Xu}\textsuperscript{1},~
\textbf{Jianbing Shen}\textsuperscript{\textdagger,1}
\\[2pt]
\footnotesize
\textsuperscript{1}SKL-IOTSC, CIS, University of Macau \quad
\textsuperscript{2}Auckland University of Technology\\
\scriptsize 
yucheng.zhou@connect.um.edu.mo
}

\begin{document}

\maketitle
\begin{abstractbox}
\begin{center}
\textbf{\Large Abstract}
\end{center}
Large Language Models (LLMs) have shown strong potential for medical reasoning, yet the scarcity and cost of expert-annotated data constrain their progress. While reinforcement learning offers a scalable alternative, standard outcome-based methods in medicine often suffer from autoregressive credit assignment failure and gradient variance explosion. This leads to the \emph{``Right Answer, Wrong Reason''} trap, where models inadvertently reinforce spurious correlations and dataset shortcuts rather than valid clinical deduction. In this work, we propose \textbf{Causally-Aligned Reasoning Exploration (CARE)}, a theoretically grounded framework for intrinsic experience curation. CARE is built upon two rigorous conditions for high-quality training trajectories: Causal Sufficiency, which utilizes an agreement-based self-verification mechanism to mimic $do$-calculus interventions and effectively debias gradients; and Proximal Learnability, which employs dynamic entropy bounds to select experiences within the model's zone of proximal development for variance-bounded optimization. These rigorously filtered experiences are optimized via a dual-stream objective that combines on-policy group-relative exploration with difficulty-weighted experience replay. Extensive experiments on diverse medical multimodal and text-only benchmarks demonstrate that CARE consistently outperforms other strong competitors, substantially reducing correct-but-inconsistent reasoning and improving training stability.

\end{abstractbox}

\begingroup
\renewcommand{\thefootnote}{\fnsymbol{footnote}}
\footnotetext[1]{Equal contribution.}
\footnotetext[2]{Corresponding author.}
\endgroup

\section{Introduction}
\label{sec:intro}

Recent advances in Large Language Models (LLMs) have significantly expanded the scope of medical artificial intelligence, enabling AI systems to interpret complex clinical data, reason over multimodal inputs, and provide diagnostic support across a wide range of medical tasks \citep{zhou2026medical,advance-in-llm}.
The dominant training paradigm for such models remains Supervised Fine-Tuning (SFT) on curated medical datasets \citep{hulu-med,lingshu}.
While effective, SFT fundamentally depends on large volumes of expert-annotated data, which are costly, time-consuming, and often infeasible to scale in high-stakes medical domains.
Consequently, there is growing interest in self-improving paradigms, particularly reinforcement learning (RL), where models optimize their reasoning strategies by learning from their own generated trajectories \citep{ReST, deepseekmath,zhou2026compatibility}.

However, directly applying outcome-based RL (e.g., GRPO) to medical reasoning poses unique and critical challenges \citep{research-in-development}.
First, unlike domains such as mathematics or programming, where correctness can often be verified deterministically \citep{rlfreason}, medical reasoning lacks inexpensive and reliable automated verifiers.
Relying solely on outcome supervision (e.g., ground-truth diagnoses) creates a severe misalignment between the reward signal and the actual reasoning process.
In this work, we formally analyze this as the \emph{``Right Answer, Wrong Reason''} trap: due to autoregressive credit assignment failure, outcome-based rewards inevitably reinforce spurious correlations and dataset shortcuts. The model learns to guess the correct label without engaging in valid clinical deduction. Such causally misaligned trajectories may artificially inflate benchmark accuracy but severely compromise the model's reliability and interpretability in real-world clinical settings.

Second, medical datasets exhibit extreme heterogeneity in information density, ranging from trivial recognition tasks to ambiguous, underspecified queries.
Standard exploration strategies, which indiscriminately reinforce all correct rollouts, fail to account for the variance of the learning signal.
As we demonstrate theoretically, unbounded sequence likelihoods destabilize the policy gradient. Trivial samples yield vanishing gradients, while highly uncertain or hallucinatory rationales introduce exploding variance that destroys optimization stability.
Efficient self-improvement requires distinguishing \emph{effective learning experiences} from statistical noise, ensuring optimization occurs strictly within a stable region of the model's competence.

To overcome these fundamental limitations, we argue that robust medical self-improvement requires shifting from simple outcome monitoring to rigorous intrinsic experience curation.
We posit that a high-quality training trajectory must satisfy two theoretical conditions, which we rigorously formulate:
(1) \textbf{Causal Sufficiency}: The generated rationale must be causally sufficient to recover the final answer. We prove that enforcing this condition acts as a \emph{gradient debiasing operator}, ensuring the decision stems from logical reasoning rather than spurious priors; and
(2) \textbf{Proximal Learnability}: The trajectory should fall within a dynamic, entropy-bounded window, the model’s \emph{zone of proximal development}. We mathematically guarantee that this constraint provides \emph{variance-bounded optimization}, maximizing gradient utility while preventing instability.

Building on these theoretical foundations, we propose \textbf{Causally-Aligned Reasoning Exploration (CARE)}, a framework designed to seamlessly translate our theoretical principles into a practical RL algorithm.
CARE replaces naive filtration with a structured admission mechanism.
To enforce causal sufficiency, we introduce an Agreement-Based Self-Verification protocol, which accepts an experience only if the model can reproduce the final decision when conditioned \emph{solely} on the generated rationale (mimicking a $do$-calculus intervention).
To control optimization variance, we implement a Learnability-Aware Exploration mechanism, dynamically selecting experiences with normalized sequence likelihoods that offer the highest information gain.
These rigorously curated experiences are then optimized via a dual-stream objective, combining on-policy group-relative exploration with difficulty-weighted experience replay.

We evaluate CARE extensively on diverse medical multimodal and text-only benchmarks, including clinical VQA and complex diagnostic reasoning, and CARE consistently outperforms other models of a similar scale.
Crucially, our analysis reveals that CARE not only improves accuracy but also significantly reduces correct-but-inconsistent reasoning, validating that enforcing causal alignment effectively mitigates shortcut learning.

Our contributions are summarized as follows:
\begin{itemize}[leftmargin=*, itemsep=2pt, topsep=2pt, partopsep=2pt, parsep=2pt]
    \item We theoretically prove that standard outcome-based RL in medicine suffers from autoregressive credit assignment failure and from an explosion of gradient variance. This formal analysis exposes the fundamental ``Right Answer, Wrong Reason'' trap that reinforces spurious correlations.
    \item We propose CARE, a framework that reformulates self-improvement as intrinsic curation governed by \emph{Causal Sufficiency}. We implement this via Agreement-Based Self-Verification, which acts as a gradient debiasing operator to eliminate shortcut learning.
    \item We introduce a \emph{Proximal Learnability} mechanism that filters trajectories based on dynamic entropy bounds. This theoretically guarantees variance-bounded optimization, ensuring the model learns solely from effective experiences within its zone of proximal development.
    \item Extensive experiments across multimodal and text-only medical benchmarks demonstrate that CARE outperforms other models of a similar scale. 
\end{itemize}

\section{Related Work}
Foundational RL methods, evolving from Policy Gradient \citep{policy_gradient_methods} and Actor-Critic \citep{actor-critic, atari} to stable algorithms like PPO \citep{ppo,zhou2025improving} (addressing TRPO \citep{trpo} instability) and sequence modeling via Decision Transformers \citep{decision_transformer, offline-rl,zhou2026world}, have fundamentally shaped LLMs. While early medical LLMs such as HealthGPT \citep{healthGPT} and MMedLM \citep{mmedlm} primarily utilized pre-training and instruction tuning \citep{refining}, the adoption of RLHF \citep{RLHF} became standard for aligning models with medical values \citep{zhang2023huatuogpt, medprm}. To overcome context limitations and enhance clinical reasoning in complex scenarios \citep{yang2023zhongjing, capo}, researchers integrated multi-agent frameworks \citep{zhou2025mam, autonomous-agent, metaGPT, chatdev, llm-powered} and memory-augmented mechanisms \citep{memGPT, camel}. Specialized reasoning methods have also emerged, including StarPO \citep{StarPO} and GRPO \citep{GRPO,song2026broad}; notably, GRPO eliminates the value model to stabilize updates and has been successfully applied to medical multimodal models \citep{2025medvlm-r1, gmai-vl-r1, zheng2026clinical} and mathematical reasoning \citep{exgrpo}. Recent advancements focus on high-fidelity simulation \citep{baichuan-m2, llm-powered}, and intermediate reasoning trajectories, exemplified by HuatuoGPT-o1 \citep{2024huatuogpt-o1} (inspired by OpenAI o1 \citep{openai}). Furthermore, RL frameworks have expanded to multimodal domains \citep{vision-r1} through models like MedCCO \citep{medcco} and Med-R1 \citep{med-r1}. Current trends emphasize practical efficiency and comprehensive capabilities, e.g., multi-agent data generation \citep{reasonmed}, ``RL + LoRA'' efficient training \citep{rarl-vlm, lora}, and all-purpose processing (Lingshu \citep{lingshu} and Hulu-Med \citep{hulu-med}).

\section{Theoretical Analysis: Causal Misalignment and Optimization Stability}
\label{sec:theory}

In this section, we provide a formal analysis of the optimization dynamics in self-improving LLMs. We first prove that standard outcome-based RL is structurally prone to reinforcing spurious correlations due to autoregressive credit assignment failure. We then demonstrate that our proposed \textbf{Causal Sufficiency} mechanism acts as a gradient debiasing operator, while \textbf{Proximal Learnability} guarantees variance-bounded optimization.

\subsection{The Autoregressive Shortcut Trap}

Let $x \in \mathcal{X}$ be a clinical query. The LLM generates a trajectory $y$ composed of a latent rationale sequence $r = (r_1, \dots, r_{T})$ and a final answer $a$. The policy $\pi_\theta$ factorizes the joint probability via the chain rule:
\begin{align}
    \pi_\theta(r, a \mid x) = \underbrace{\pi_\theta(r \mid x)}_{\text{Rationale Generation}} \cdot \underbrace{\pi_\theta(a \mid x, r)}_{\text{Answer Prediction}}.
\end{align}
In standard outcome-based RL, the objective is to maximize the expected return $J(\theta) = \mathbb{E}_{y \sim \pi_\theta} [\mathcal{U}(a, a^*)]$, where $\mathcal{U}(a, a^*) = \mathbb{I}[a = a^*]$ is the binary correctness reward. The gradient estimator is:
\begin{align}
\label{eq:standard_pg}
    \nabla_\theta J(\theta) \approx \mathbb{E} \left[ \mathcal{U}(a, a^*) \left( \nabla_\theta \log \pi_\theta(r \mid x) + \nabla_\theta \log \pi_\theta(a \mid x, r) \right) \right].
\end{align}

\begin{definition}[Spurious Shortcut]\label{def:shortcut}
A rationale $r_S$ is a \emph{spurious shortcut} if it is logically insufficient to deduce $a^*$ (i.e., $a^* \not\perp \perp x \mid r_S$), but the model correctly predicts $a^*$ due to residual correlations between $x$ and $a^*$ in the pre-trained weights.
\end{definition}

\begin{theorem}[Credit Assignment Failure]
\label{thm:failure}
Under standard outcome supervision, if a spurious shortcut $r_S$ leads to a correct answer $a^*$, the policy gradient strictly increases the likelihood of $r_S$:
\begin{align}
    \mathbb{E} \left[ \nabla_\theta \log \pi_\theta(r_S \mid x) \cdot \mathbb{I}[a = a^*] \right] > 0.
\end{align}
\end{theorem}
\begin{proof}
The scalar reward $\mathcal{U}=1$ is broadcast to the entire sequence. Since $\nabla_\theta \log \pi_\theta(r_S \mid x)$ is additively coupled with the answer's gradient, the optimizer cannot distinguish whether the answer was derived \emph{via} $r_S$ or simply predicted \emph{after} $r_S$. Consequently, valid reasoning and hallucinatory shortcuts are indiscriminately reinforced, cementing the ``Right Answer, Wrong Reason'' behavior.
\end{proof}

\subsection{Causal Sufficiency as Gradient Debiasing}

To rectify this, we must ensure that the rationale $r$ is a \emph{sufficient statistic} for the answer $a$. We formalize this using the $do$-calculus intervention $do(R=r)$, which severs the dependence on the input $x$.

\begin{definition}[Causal Sufficiency]
A trajectory $(r, a)$ satisfies causal sufficiency if the answer $a$ is recoverable from $r$ alone under the current policy:
\begin{align}
    \phi_{\text{causal}}(r, a) = \mathbb{I} \left[ \arg\max_{\tilde{a}} \pi_{\bar{\theta}}(\tilde{a} \mid do(R=r)) = a \right],
\end{align}
where $\pi_{\bar{\theta}}$ represents the model in inference mode without access to $x$.
\end{definition}

\begin{theorem}[Gradient Debiasing]
\label{thm:debiasing}
Let $\mathcal{S}$ be the subspace of spurious shortcuts and $\mathcal{C}$ be the subspace of valid causal reasoning. Applying the filter $\phi_{\text{causal}}$ eliminates gradient contributions from $\mathcal{S}$:
\begin{align}
    \forall r_S \in \mathcal{S}, \quad \nabla_\theta J_{\text{CARE}}(\theta) \big|_{r=r_S} \to 0.
\end{align}
\end{theorem}
\begin{proof}
For any $r_S \in \mathcal{S}$, the correctness of $a^*$ relies on the confounder $x$. Under the intervention $do(R=r_S)$, $x$ is masked. Since $r_S$ lacks logical sufficiency, the conditional probability $\pi(a^* \mid r_S)$ drops significantly compared to $\pi(a^* \mid x, r_S)$, causing the verification check to fail ($\phi_{\text{causal}}=0$). The gradient for this trajectory is zeroed out, effectively debiasing the update direction towards $\mathcal{C}$.
\end{proof}

\subsection{Proximal Learnability and Variance Control}

Finally, we address the optimization stability. The high variance of medical query difficulty leads to unstable gradients. Let $\mathcal{L}_{\text{seq}}(y) = -\frac{1}{|y|} \log \pi_\theta(y|x)$ be the length-normalized negative log-likelihood (NLL).

\begin{lemma}[Gradient Variance Bound]
\label{lem:variance}
The variance of the policy gradient estimator $\hat{g}$ is bounded by the second moment of the sequence NLL:
\begin{align}
    \text{Var}(\hat{g}) \le \mathcal{O} \left( \mathbb{E} \left[ \mathcal{L}_{\text{seq}}(y)^2 \right] \right).
\end{align}
\end{lemma}

\begin{theorem}[Variance-Bounded Optimization]
\label{thm:stability}
By restricting training to a dynamic window $\mathcal{W} = [\tau_{\text{low}}, \tau_{\text{high}}]$ of NLLs (Proximal Learnability), we guarantee:
\begin{enumerate}
    \item \textbf{Bounded Gradient Variance:} $\text{Var}(\hat{g} | y \in \mathcal{W}) \le C \cdot \tau_{\text{high}}^2$, preventing explosion from hallucinations.
    \item \textbf{Non-Vanishing Signal:} $||\mathbb{E}[\hat{g} | y \in \mathcal{W}]|| \ge \epsilon > 0$, ensuring informative updates by rejecting trivial samples.
\end{enumerate}
\end{theorem}

Theorem \ref{thm:stability} proves that our adaptive filtering does not merely curate data quality but mathematically constrains the optimization trajectory within a stable, low-variance region, enabling efficient self-improvement.

\section{Methodology}
\label{sec:method}

\begin{figure*}[t]
    \centering
    \includegraphics[width=1\linewidth]{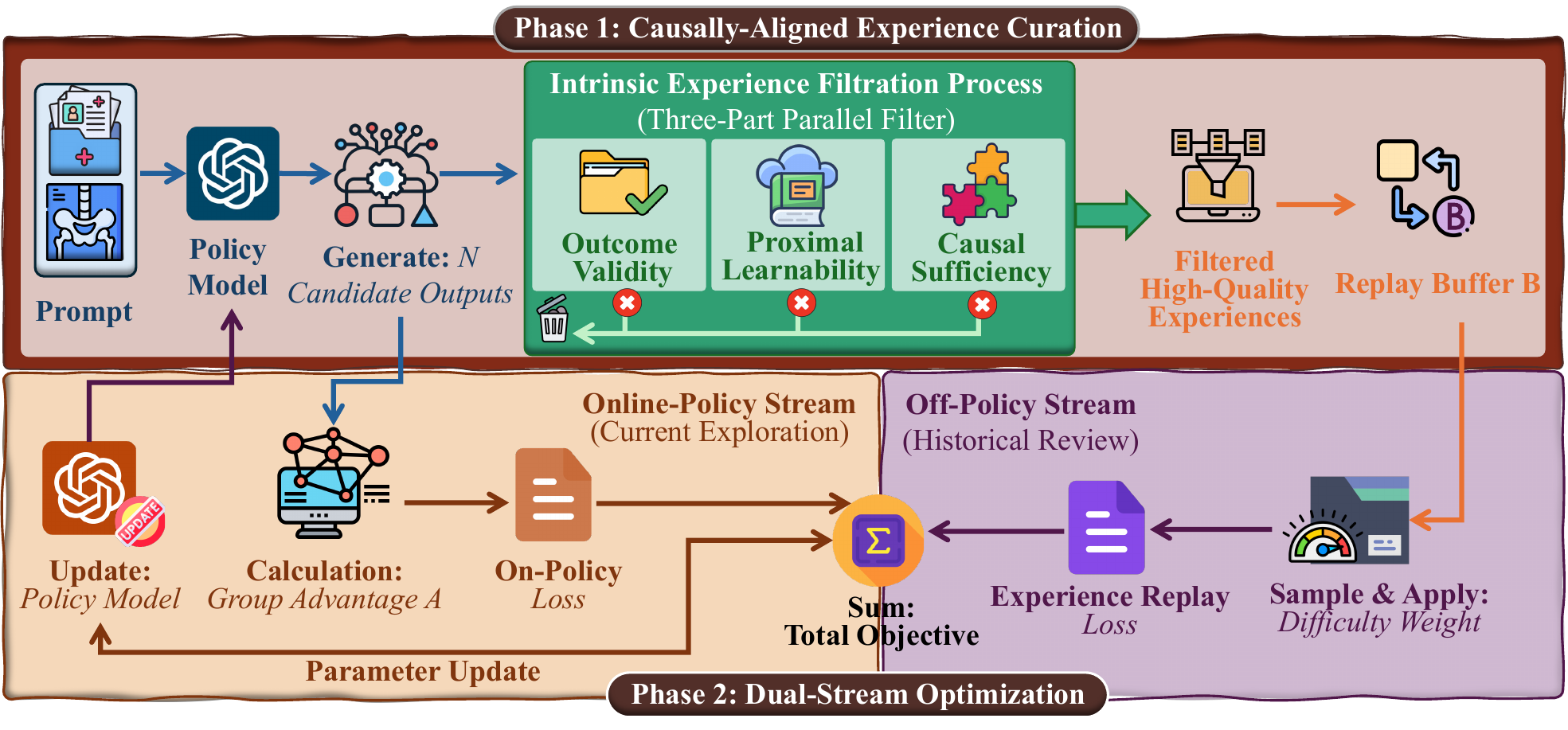}
    \caption{\small \textbf{Overview of the proposed CARE framework.} The process consists of two phases: (1) \textbf{Causally-Aligned Experience Curation}, where candidate rollouts are evaluated through a three-part filter (Outcome Validity, Proximal Learnability, and Causal Sufficiency) to construct a high-fidelity learning signal; and (2) \textbf{Dual-Stream Optimization}, which robustly updates the policy using both on-policy exploration advantages and difficulty-weighted replay.}
    \label{fig:model}
\end{figure*}

We present Causally-Aligned Reasoning Exploration (CARE), a self-improvement framework designed to align policy optimization with valid clinical reasoning. 
As illustrated in Figure \ref{fig:model}, CARE translates the theoretical principles derived in Section \ref{sec:theory} into a practical training algorithm. 
We shift the training paradigm from naive outcome supervision to intrinsic experience curation. 
We first detail our structured admission mechanism, which dynamically filters trajectories based on Proximal Learnability (variance control) and Causal Sufficiency (gradient debiasing). 
Finally, we describe the dual-stream optimization objective that leverages these curated experiences to stably improve the policy.

\subsection{Problem Formulation}
We consider a multimodal medical policy $\pi_\theta$ parameterized by $\theta$, which processes a clinical query $x = (x_{\text{img}}, x_{\text{txt}})$ to generate a response $y$. 
To enable interpretability, we enforce a structured chain-of-thought format $y = (r, a)$, where $r$ is the generated rationale and $a$ is the final diagnostic answer. 
The training set $\mathcal{D}$ contains query-answer pairs $(x, a^*)$, but lacks expert-annotated rationales.

Our objective is to maximize the expected return $J(\theta) = \mathbb{E}_{x \sim \mathcal{D}, y \sim \pi_\theta} [R(x, y)]$. 
As proved in Theorem \ref{thm:failure}, relying solely on a binary outcome reward $R(x, y) = \mathbb{I}[a = a^*]$ induces a gradient leak that reinforces spurious shortcuts. 
CARE addresses this by collecting $N$ rollouts $\mathcal{Y}_x = \{y^{(i)}\}_{i=1}^N$ per prompt and rigorously curating them to construct an augmented reward signal and a high-fidelity replay buffer $\mathcal{B}$.

\subsection{Causally-Aligned Experience Curation}
\label{sec:curation}

To mitigate the reinforcement of spurious correlations and stabilize optimization, we introduce a unified admission function $\Phi(x, y): \mathcal{X} \times \mathcal{Y} \to \{0, 1\}$. 
A trajectory is admitted if and only if it satisfies three hierarchical criteria.

\paragraph{Outcome Validity.}
The prerequisite for any valid experience is that the final decision aligns with the ground truth. 
Let $\mathcal{M}(\cdot)$ be a rule-based answer extraction function:
\begin{align}
    \phi_{\text{valid}}(x, y) = \mathbb{I}\left[ \mathcal{M}(y) = a^* \right].
\end{align}

\paragraph{Proximal Learnability (Variance Control).}
As theoretically established in Lemma \ref{lem:variance}, unbounded sequence negative log-likelihood (NLL) directly leads to gradient variance explosion. 
To ensure optimization strictly occurs within the model's \emph{zone of proximal development} (Theorem \ref{thm:stability}), we quantify trajectory uncertainty using the length-normalized NLL:
\begin{align}
    \mathcal{L}_{\text{seq}}(y|x; \theta) = -\frac{1}{|y|} \sum_{t=1}^{|y|} \log \pi_\theta(y_t \mid x, y_{<t}).
\end{align}
We maintain a moving history buffer $\mathcal{H}$ of recent NLL values to dynamically compute a learnability window $[\tau_{\text{low}}, \tau_{\text{high}}]$ based on its quantiles. The learnability filter is:
\begin{align}
    \phi_{\text{learn}}(y) = \mathbb{I}\left[ \tau_{\text{low}} \le \mathcal{L}_{\text{seq}}(y|x; \theta) \le \tau_{\text{high}} \right].
\end{align}
This effectively rejects trivial shortcuts (vanishing gradients) and chaotic hallucinations (exploding variance), maximizing the stable Fisher information gain per update.

\paragraph{Causal Sufficiency via Agreement Verification.}
To explicitly debias the gradient (Theorem \ref{thm:debiasing}), we must ensure the decision $a$ stems causally from the reasoning $r$. 
We empirically instantiate the $do(R=r)$ intervention using Agreement-Based Self-Verification. 
For a generated pair $(r, a)$, we construct a verification prompt $\mathcal{T}(r)$ (e.g., \emph{``Based solely on the reasoning: $r$, what is the diagnosis?''}) that masks the original query $x$. We then query the model in inference mode ($\pi_{\bar{\theta}}$) to predict a new answer $\hat{a}$:
\begin{align}
    \hat{a} = \operatorname*{argmax}_{\tilde{a}} \pi_{\bar{\theta}}(\tilde{a} \mid \mathcal{T}(r)).
\end{align}
The trajectory is causally sufficient only if the rationale is robust enough to independently reproduce the answer:
\begin{align}
    \phi_{\text{causal}}(r, a) = \mathbb{I}\left[ \hat{a} = a \right].
\end{align}
By penalizing trajectories where $\hat{a} \neq a$, this mechanism guarantees that $r$ is a sufficient statistic for $a$, severing the dependence on dataset artifacts.

\paragraph{Unified Admission.}
The final CARE admission function serves as both the augmented RL reward and the gatekeeper for the replay buffer $\mathcal{B}$:
\begin{align}
    \Phi(x, y) = \phi_{\text{valid}}(x, y) \cdot \phi_{\text{learn}}(y) \cdot \phi_{\text{causal}}(r, a).
\end{align}

\subsection{Dual-Stream Optimization}

We optimize $\pi_\theta$ using a hybrid objective that combines on-policy Group Relative Policy Optimization (GRPO) with difficulty-weighted experience replay.

\paragraph{On-Policy Group Relative Exploration.}
For each prompt $x$, we generate a group of $N$ candidate rollouts $\mathcal{Y}_x$. 
Instead of relying on a separate value network, we compute the advantage $A^{(i)}$ for the $i$-th rollout using the unified admission score $\Phi^{(i)} = \Phi(x, y^{(i)})$, standardized within the group:
\begin{align}
    A^{(i)} = \frac{\Phi^{(i)} - \mu(\Phi)}{\sigma(\Phi) + \epsilon},
\end{align}
where $\mu(\Phi)$ and $\sigma(\Phi)$ are the mean and standard deviation of the admission scores in $\mathcal{Y}_x$. 
The on-policy objective is:
\begin{align}
    \mathcal{L}_{\text{on}}(\theta) = - \mathbb{E}_{x \sim \mathcal{D}} \left[ \frac{1}{N} \sum_{i=1}^N \left( A^{(i)} \log \pi_\theta(y^{(i)}|x) - \beta \mathbb{D}_{\text{KL}}[\pi_\theta || \pi_{\text{ref}}] \right) \right],
\end{align}
where the KL divergence penalty prevents excessive deviation from the reference policy $\pi_{\text{ref}}$.

\paragraph{Difficulty-Weighted Experience Replay.}
Simultaneously, we sample high-quality trajectories $(x, y) \sim \mathcal{B}$ that have successfully passed the curation pipeline. 
To further prioritize samples at the frontier of the model's capability, we apply an importance weight $w(y) \propto \mathcal{L}_{\text{seq}}(y|x; \theta)$ (normalized within the batch). The replay loss is:
\begin{align}
    \mathcal{L}_{\text{rep}}(\theta) = - \mathbb{E}_{(x, y) \sim \mathcal{B}} \left[ w(y) \log \pi_\theta(y|x) \right].
\end{align}

\paragraph{Total Objective.}
The final training objective seamlessly integrates exploration and stable exploitation:
\begin{align}
    \mathcal{L}_{\text{CARE}}(\theta) = \mathcal{L}_{\text{on}}(\theta) + \lambda \mathcal{L}_{\text{rep}}(\theta),
\end{align}
where $\lambda$ controls the replay mixing ratio.

\subsection{Training Algorithm}
The CARE training procedure is summarized in Algorithm \ref{alg:care}. 
The framework alternates between generating rollout groups, filtering them for causal sufficiency and learnability, and updating the policy via the dual-stream objective.

\begin{algorithm}[ht]
\caption{\small Causally-Aligned Reasoning Exploration (CARE)}
\label{alg:care}
\begin{algorithmic}[1]
   \STATE {\bfseries Input:} Dataset $\mathcal{D}$, Policy $\pi_\theta$, Reference $\pi_{\text{ref}}$, Buffer $\mathcal{B} \leftarrow \emptyset$, History $\mathcal{H} \leftarrow \emptyset$.
   \WHILE{not converged}
       \STATE Sample a batch of prompts $\{x_j\}$ from $\mathcal{D}$.
       \STATE Update dynamic window $[\tau_{\text{low}}, \tau_{\text{high}}]$ using quantiles of $\mathcal{H}$.
       \FOR{each prompt $x_j$}
           \STATE Generate $N$ rollouts $\{y_j^{(i)}\}_{i=1}^N \sim \pi_\theta(\cdot|x_j)$.
           \FOR{$i=1$ {\bfseries to} $N$}
               \STATE Compute NLL $\ell_j^{(i)} \leftarrow \mathcal{L}_{\text{seq}}(y_j^{(i)})$. Update $\mathcal{H}$.
               \STATE \textcolor{gray}{// Check Validity and Learnability}
               \IF{$\phi_{\text{valid}}(y_j^{(i)}) \land \phi_{\text{learn}}(y_j^{(i)})$}
                   \STATE \textcolor{gray}{// Check Causal Sufficiency via self-verification}
                   \STATE $\hat{a} \leftarrow \operatorname{argmax} \pi_{\bar{\theta}}(a|\mathcal{T}(r_j^{(i)}))$.
                   \IF{$\hat{a} == a_j^{(i)}$}
                       \STATE $\Phi_j^{(i)} \leftarrow 1$, Add $(x_j, y_j^{(i)})$ to $\mathcal{B}$.
                   \ELSE
                       \STATE $\Phi_j^{(i)} \leftarrow 0$.
                   \ENDIF
               \ELSE
                   \STATE $\Phi_j^{(i)} \leftarrow 0$.
               \ENDIF
           \ENDFOR
           \STATE Compute advantages $A_j^{(i)}$ using $\{\Phi_j^{(i)}\}_{i=1}^N$.
       \ENDFOR
       \STATE Sample replay batch $\mathcal{B}_{\text{batch}} \sim \mathcal{B}$.
       \STATE Update $\theta$ by minimizing $\mathcal{L}_{\text{CARE}} = \mathcal{L}_{\text{on}} + \lambda \mathcal{L}_{\text{rep}}$.
   \ENDWHILE
\end{algorithmic}
\end{algorithm}

\section{Experiments}

\subsection{Experimental Settings}
\paragraph{Implementation Details.}
Our model, CARE, is built upon the HULU 7B medical vision-language backbone \citep{hulu-med}. To isolate the effect of CARE from the backbone architecture, we additionally report CARE (HuatuoGPT-V), which applies CARE to the HuatuoGPT-V 7B medical VLM \citep{huatuogptvision}. Training is conducted on a mixture of medical datasets, including PMC-VQA \citep{PMC-VQA}, SLAKE \citep{SLAKE}, PathVQA \citep{PathVQA}, MedMCQA \citep{medmcqa}, and PubMedQA \citep{pubmedqa}. These datasets cover a broad spectrum of clinical imaging and professional knowledge reasoning tasks. We strictly ensure that there is no overlap between the training samples and the test sets of our evaluation benchmarks to prevent data contamination.

CARE is implemented using a rollout-based reinforcement learning framework on $8 \times$ A100 GPUs with a rollout batch size of 128 and an update batch size of 64. For each prompt, we generate $N=8$ on-policy rollouts. The dynamic learnability window is governed by quantile thresholds $\alpha=0.2$ and $\beta=0.9$, calculated using a FIFO-based moving history buffer $\mathcal{H}$ of 2,000 NLL values. We set the replay loss weight $\lambda=1.0$, the KL divergence penalty $\beta=0.04$, and a fixed 50\% replay ratio. The importance weight $w(y)$ is Softmax-normalized within each replay batch based on length-normalized NLLs. To ensure stable initialization, experience-based optimization and the replay buffer are activated only after the batch Pass@1 accuracy reaches 35\%.

\paragraph{Benchmarks and Evaluation.}
We evaluate CARE across a comprehensive suite of clinical imaging and reasoning tasks, including multimodal benchmarks (OmniVQA \citep{omnivqa}, PMC-VQA \citep{PMC-VQA}, VQA-RAD \citep{vqa-rad}, MMMU-Med \citep{mmmu}, PathVQA \citep{PathVQA}, MedXQA \citep{medxqa}, and SLAKE \citep{SLAKE}) and professional text benchmarks (MMLU-Pro-Med \citep{mmlupro}, PubMedQA \citep{pubmedqa}, MedMCQA \citep{medmcqa}, MedQA \citep{medqa}, and MMLU-Med \citep{mmlu}). To ensure reproducibility, all evaluation protocols, scoring scripts, and metrics strictly follow the official Hulu-Med GitHub repository.

\subsection{Main Results}
\begin{table*}[t]\small
\centering
\setlength{\tabcolsep}{3pt}
\caption{\small Performance comparison on medical multimodal benchmarks across proprietary, general-purpose, and medical VLMs.}
\label{tab:multimodal_main}
\resizebox{\linewidth}{!}{
\begin{tabular}{lccccccc}
\toprule
Models & OmniVQA & PMC-VQA & VQA-RAD & SLAKE & PathVQA & MedXQA & MMMU-Med \\
\midrule
\multicolumn{8}{l}{\textbf{Proprietary Models}} \\
GPT-4.1 \citep{gpt4.1} & 75.5 & 55.2 & 65.0 & 72.2 & 55.5 & 45.2 & 75.2 \\
GPT-4o \citep{gpt4o} & 67.5 & 49.7 & 61.0 & 71.2 & 55.5 & 44.3 & 62.8 \\
Claude Sonnet 4 \citep{claude4} & 65.5 & 54.4 & 67.6 & 70.6 & 54.2 & 43.3 & 74.6 \\
Gemini-2.5-Flash \citep{gemini2.5} & 71.0 & 55.4 & 68.5 & 75.8 & 55.4 & 52.8 & 76.9 \\
\midrule
\multicolumn{8}{l}{\textbf{General-purpose MM}} \\
Qwen2.5VL-7B \citep{qwen2.5vl} & 63.6 & 51.9 & 63.2 & 66.8 & 44.1 & 20.1 & 50.6 \\
Janus-Pro-7B \citep{januspro} & 59.6 & 50.1 & 49.7 & 55.2 & 35.4 & 18.4 & 36.1 \\
InternVL2.5-8B \citep{intervl2.5} & 81.3 & 51.3 & 59.4 & 69.0 & 42.1 & 21.7 & 53.5 \\
InternVL3-8B \citep{internvl3} & 79.1 & 53.8 & 65.4 & 72.8 & 48.6 & 22.4 & 59.2 \\
\midrule
\multicolumn{8}{l}{\textbf{Medical MM}} \\
BiomedGPT \citep{biomedgpt} & 27.9 & 27.6 & 16.6 & 13.6 & 11.3 & - & 24.9 \\
Med-R1-2B \citep{med-r1} & - & 47.4 & 39.0 & 54.5 & 15.3 & 21.1 & 34.8 \\
MedVLM-R1-2B \citep{2025medvlm-r1} & 77.6 & 48.8 & 49.2 & 56.3 & 36.0 & 21.4 & 35.2 \\
HealthGPT-M3 \citep{healthGPT} & 71.5 & 55.4 & 56.8 & 70.8 & 55.4 & 22.4 & 42.8 \\
BioMedX2-8B \citep{bimedix2} & 66.0 & 41.8 & 55.7 & 54.1 & 34.6 & 21.9 & 39.8 \\
LLaVA-Med-7B \citep{llavamed} & 34.8 & 22.7 & 46.6 & 51.9 & 35.2 & 20.8 & 28.1 \\
MedGemma-4B-IT \citep{medgemma} & 70.7 & 49.2 & 72.3 & 78.2 & 48.1 & 25.4 & 43.2 \\
HuatuoGPT-V 7B \citep{huatuogptvision} & 74.3 & 53.1 & 67.6 & 68.1 & 44.8 & 23.2 & 49.8 \\
Lingshu-7B \citep{lingshu} & 82.9 & 56.3 & 67.9 & 83.1 & 61.9 & 26.7 & - \\
Hulu-Med-7B \citep{hulu-med} & 84.2 & 66.8 & 78.0 & 86.8 & 65.6 & 29.0 & 51.4 \\
\rowcolor{gray!15} CARE-7B (HuatuoGPT-V) & 75.8 & 54.6 & 69.0 & 69.8 & 46.3 & 25.4 & 51.2 \\
\rowcolor{gray!15} CARE-7B & \textbf{85.6} & \textbf{68.2} & \textbf{79.3} & \textbf{88.1} & \textbf{67.1} & \textbf{31.2} & \textbf{53.0} \\
\bottomrule
\end{tabular}}
\end{table*}

\begin{table*}[t]\small
\centering
\setlength{\tabcolsep}{3pt}
\caption{\small Performance comparison on medical text benchmarks across proprietary, general-purpose, and medical models.}
\label{tab:text_main}
\resizebox{\linewidth}{!}{
\begin{tabular}{lcccccc}
\toprule
Models & MMLU-Pro-Med & MedXQA & PubMedQA & MedMCQA & MedQA & MMLU-Med \\
\midrule
\multicolumn{7}{l}{\textbf{Proprietary Models}} \\
GPT-4.1 \citep{gpt4.1} & 78.0 & 30.9 & 75.6 & 77.7 & 89.1 & 89.6 \\
o3-mini \citep{openai2025o3mini} & 78.1 & 35.4 & 73.6 & 60.6 & 74.5 & 87.0 \\
GPT-4o \citep{gpt4o} & 75.6 & 25.9 & 71.8 & 76.9 & 89.2 & 88.2 \\
Claude Sonnet 4 \citep{claude4} & 79.5 & 33.6 & 78.6 & 79.3 & 92.1 & 91.3 \\
Gemini-2.5-Flash \citep{gemini2.5} & 70.0 & 35.6 & 73.8 & 73.6 & 91.2 & 84.2 \\
Deepseek-V3 \citep{deepseekv3} & 74.6 & 20.0 & 77.7 & 88.0 & 51.0 & 86.5 \\
\midrule
\multicolumn{7}{l}{\textbf{General-purpose MM}} \\
Qwen2.5VL-7B \citep{qwen2.5vl} & 50.5 & 12.8 & 76.4 & 52.6 & 57.3 & 73.4 \\
Janus-Pro-7B \citep{januspro} & 20.2 & 10.0 & 72.0 & 37.5 & 37.4 & 46.4 \\
InternVL2.5-8B \citep{intervl2.5} & 50.6 & 11.6 & 76.4 & 52.4 & 53.7 & 74.2 \\
InternVL3-8B \citep{internvl3} & 57.9 & 13.1 & 75.4 & 57.7 & 62.1 & 77.5 \\
\midrule
\multicolumn{7}{l}{\textbf{Medical MM}} \\
MedVLM-R1-2B \citep{2025medvlm-r1} & 24.9 & 11.8 & 66.4 & 39.7 & 42.3 & 51.8 \\
BioMedX2-8B \citep{bimedix2} & 40.8 & 13.4 & 75.2 & 52.9 & 58.9 & 68.6 \\
MedGemma-4B-IT \citep{medgemma} & 38.6 & 12.8 & 72.2 & 52.2 & 56.2 & 66.7 \\
HealthGPT-M3 \citep{healthGPT} & 38.3 & 11.5 & 57.8 & 54.2 & 55.0 & 72.5 \\
LLaVA-Med-7B \citep{llavamed} & 16.6 & 9.9 & 26.4 & 39.4 & 42.0 & 50.6 \\
HuatuoGPT-V 7B \citep{huatuogptvision} & 44.6 & 10.1 & 72.8 & 51.2 & 52.9 & 69.3 \\
Lingshu-7B \citep{lingshu} & 50.4 & 16.5 & 76.6 & 55.9 & 63.3 & 74.5 \\
Hulu-Med-7B \citep{hulu-med} & 60.6 & 19.6 & 77.4 & 67.6 & 73.5 & 79.5 \\
\rowcolor{gray!15}CARE-7B (HuatuoGPT-V) & 46.3 & 12.2 & 73.9 & 53.0 & 54.6 & 71.0 \\
\rowcolor{gray!15}CARE-7B & \textbf{62.4} & \textbf{22.1} & \textbf{78.6} & \textbf{69.1} & \textbf{75.0} & \textbf{81.1} \\
\bottomrule
\end{tabular}}
\end{table*}

\paragraph{Medical Multimodal Benchmarks.}
Table \ref{tab:multimodal_main} presents a comprehensive performance comparison across proprietary, general-purpose, and medical Vision-Language Models. 
CARE achieves the strongest overall performance among medical VLMs, consistently outperforming prior domain-specific models across all evaluated benchmarks. 
Notably, CARE significantly surpasses its backbone, Hulu-Med-7B, on all multimodal datasets. This gap demonstrates that our causally-aligned experience curation and dual-stream optimization provide substantial reasoning gains that go beyond mere architectural improvements. 
Furthermore, applying the CARE framework to HuatuoGPT-V leads to consistent performance boosts, confirming that CARE is backbone-agnostic and can effectively enhance the reasoning capabilities of existing medical VLMs. 
While large proprietary models remain competitive on certain benchmarks, CARE narrows the gap considerably.

\paragraph{Medical Text Benchmarks.}
We further evaluate the effectiveness of CARE on medical text-only reasoning benchmarks, as shown in Table \ref{tab:text_main}. 
CARE consistently outperforms strong medical baselines, including HuatuoGPT-V and Hulu-Med, across all evaluated tasks. 
These results indicate that CARE enhances not only multimodal perception but also general medical knowledge deduction. 
Specifically, the performance gains on complex benchmarks like MMLU-Med and MedQA suggest that enforcing causal sufficiency and proximal learnability effectively mitigates shortcut learning, thereby enhancing long-horizon reasoning and decision consistency in purely textual clinical settings.

\subsection{Ablation Studies and Analysis}
\label{sec:analysis}

\paragraph{Methodological Baselines and Ablation Study.}
\label{sec:ablation}
\begin{table}[t]\small
\centering
\caption{\small \textbf{Ablation Study and Baseline Comparison.} Evaluation of different training paradigms starting from the Hulu-Med-7B base model.}
\label{tab:ablation}
\begin{tabular}{lcccc}
\toprule
\textbf{Method} & \textbf{PMC-VQA} & \textbf{MedQA} & \textbf{MMMU-Med} & \textbf{Avg.} \\
\midrule
\multicolumn{5}{l}{\textit{Methodological Baselines}} \\
Hulu-Med-7B (Base / SFT) & 66.8 & 73.5 & 51.4 & 63.9 \\
+ Standard GRPO (Outcome-only) & 67.2 & 73.8 & 51.8 & 64.3 \\
\midrule
\multicolumn{5}{l}{\textit{CARE Ablations (Starting from Base)}} \\
CARE w/o Causal Sufficiency ($\phi_{\text{causal}}$) & 67.5 & 74.2 & 52.1 & 64.6 \\
CARE w/o Proximal Learnability ($\phi_{\text{learn}}$)& 67.8 & 74.4 & 52.4 & 64.9 \\
CARE w/o Difficulty Replay ($\lambda=0$) & 68.0 & 74.7 & 52.6 & 65.1 \\
\midrule
\rowcolor{gray!15}\textbf{CARE (Ours)} & \textbf{68.2} & \textbf{75.0} & \textbf{53.0} & \textbf{65.4} \\
\bottomrule
\end{tabular}
\end{table}
We first compare CARE against the standard RL baseline (Standard GRPO) and dissect our framework by systematically removing its core components. Table \ref{tab:ablation} presents the results on three representative benchmarks: PMC-VQA (multimodal), MedQA (textual), and MMMU-Med (complex reasoning). The Hulu-Med-7B model serves as our baseline.
As shown in Table \ref{tab:ablation}, applying standard outcome-based GRPO provides only marginal gains (+0.4\% average) over the strong SFT baseline, suffering from noisy reward signals. Among our proposed components, removing Causal Sufficiency ($\phi_{\text{causal}}$) results in the most noticeable performance drop compared to the full model, confirming that preventing shortcut learning is crucial for high-level reasoning. Removing Proximal Learnability ($\phi_{\text{learn}}$) also degrades performance, indicating that unconstrained exploration on excessively hard or trivial samples introduces harmful gradient noise. Finally, the dual-stream difficulty replay is essential for ensuring stable exploitation, which further elevates performance.

\paragraph{Mitigating Shortcut Learning (Right Answer, Wrong Reason).}
To empirically validate that CARE mitigates the ``Right Answer, Wrong Reason'' trap, we conducted a clinical consistency evaluation. We randomly sampled 500 queries from the MedQA test set where \emph{both} Standard GRPO and CARE predicted the correct final answer. The generated rationales were then evaluated by a panel of three human experts through majority voting, classifying them into: (1) Valid Causal Reasoning, (2) Spurious Shortcut (correct answer but logically insufficient or relying on statistical priors), and (3) Unrelated Hallucination. As shown in Figure \ref{fig:shortcut}(Left), although Standard GRPO yields correct outcomes on these samples, \textbf{45.0\%} of its rationales rely on spurious correlations. In contrast, CARE increases the proportion of Valid Causal Reasoning to \textbf{86.2\%}. This corroborates Theorem \ref{thm:debiasing}, demonstrating that our self-verification mechanism successfully severs the model's reliance on dataset artifacts in favor of causally-aligned deduction.

\begin{figure}[t]
    \centering
    \includegraphics[width=0.28\linewidth]{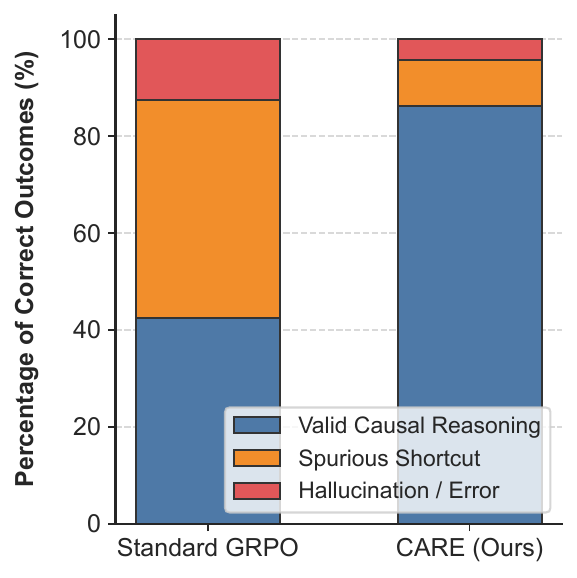}
    \includegraphics[width=0.38\linewidth]{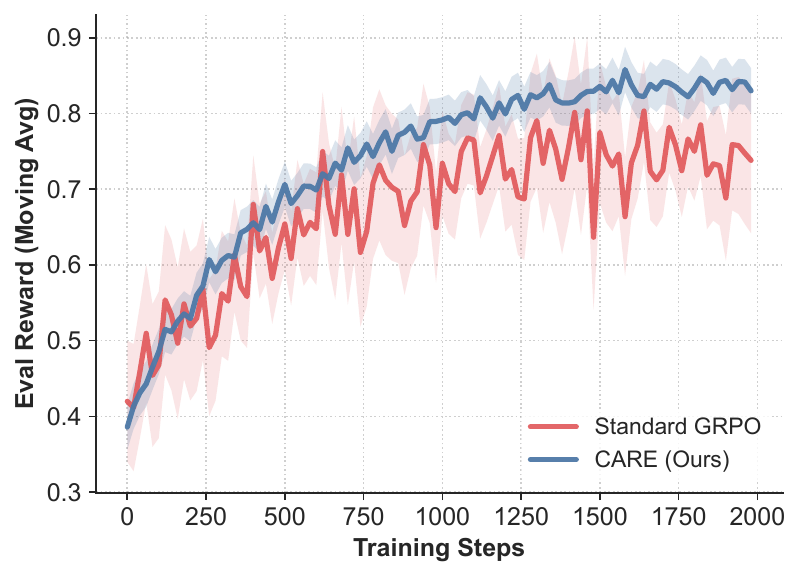}
    \includegraphics[width=0.31\linewidth]{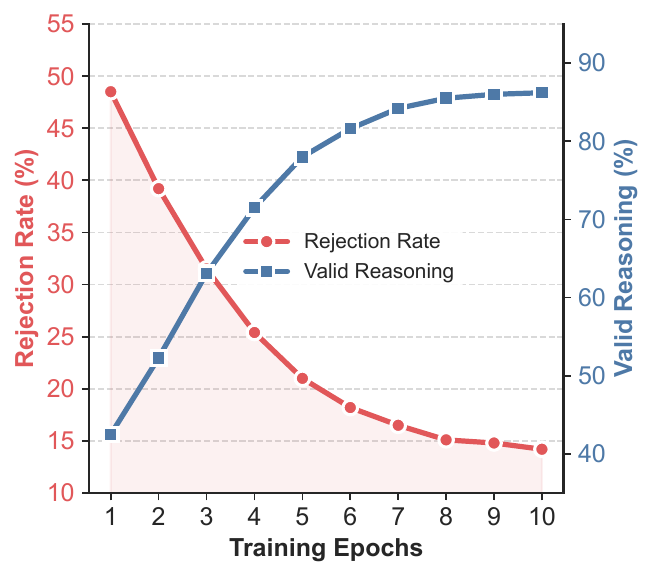}
    \caption{\small \textbf{(Left)} Reasoning Consistency Analysis: Distribution of reasoning validity for samples where the final answer was predicted correctly. \textbf{(Middle)} Training Dynamics of Standard GRPO and CARE. \textbf{(Right)} Dynamics of Causal Alignment: Self-Verification Rejection Rate (red circles) and Valid Causal Reasoning (blue squares).}
    \label{fig:shortcut}
    \label{fig:stability}
    \label{fig:verification}
\end{figure}

\paragraph{Optimization Stability via Proximal Learnability.}
In Section \ref{sec:theory}, we proved that bounding the sequence NLL stabilizes gradient variance (Theorem \ref{thm:stability}). Figure \ref{fig:stability}(Middle) visualizes this by plotting the Pass@1 evaluation accuracy on a held-out validation batch. Standard GRPO exhibits significant instability, marked by sharp spikes and reward degradation, as it indiscriminately reinforces low-probability hallucinatory trajectories that coincidentally lead to correct answers. In contrast, CARE’s Proximal Learnability filter ($\phi_{\text{learn}}$) rejects both trivial (vanishing gradient) and chaotic (exploding variance) samples. This constrains the policy to a stable region, ensuring monotonic improvement and a higher convergence ceiling.

\paragraph{Dynamics of Causal Self-Verification.}
We evaluate our agreement-based self-verification mechanism ($\phi_{\text{causal}}$) to ensure it is not excessively restrictive. Figure \ref{fig:verification} (Right) tracks the rejection rate across training epochs. Initially, 48.5\% of correct trajectories are rejected, confirming that early-stage models heavily rely on superficial shortcuts rather than logical sufficiency. As CARE reinforces causally aligned trajectories, this rejection rate steadily declines to 15\%, while valid causal reasoning simultaneously climbs to 86.2\% (Figure \ref{fig:verification}(Right), right axis). This inverse correlation shows that the model successfully internalizes structural causal constraints, shifting its behavior from dataset guessing to robust clinical deduction.

\paragraph{Empirical Validation of Gradient Variance (Theorem \ref{thm:stability}).}
We empirically validate Theorem \ref{thm:stability} by tracking policy gradient $L_2$ norms during training. As shown in Figure \ref{fig:gradient_variance}(Left), Standard GRPO suffers from frequent spikes due to indiscriminately updating on ``chaotic'' hallucinatory trajectories with unbounded NLLs that coincidentally yield correct answers. Conversely, CARE’s Proximal Learnability filter ($\phi_{\text{learn}}$) effectively rejects these variance-exploding samples. The resulting gradient $L_2$ norm remains tightly bounded and smooth, providing empirical proof for Theorem \ref{thm:stability} and explaining the superior training stability observed in our experiments.

\begin{figure}[t]
    \centering
    \includegraphics[width=0.30\linewidth]{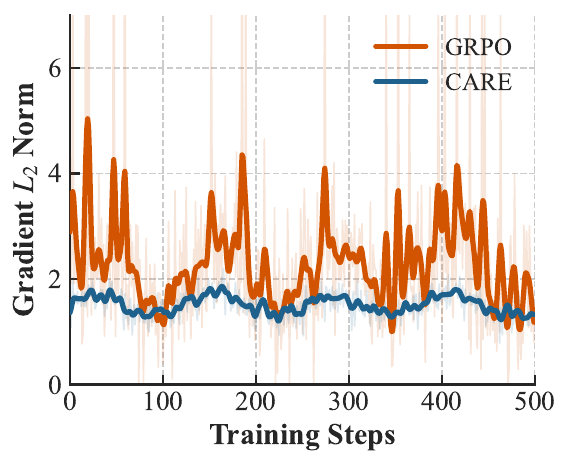}
    \includegraphics[width=0.375\linewidth]{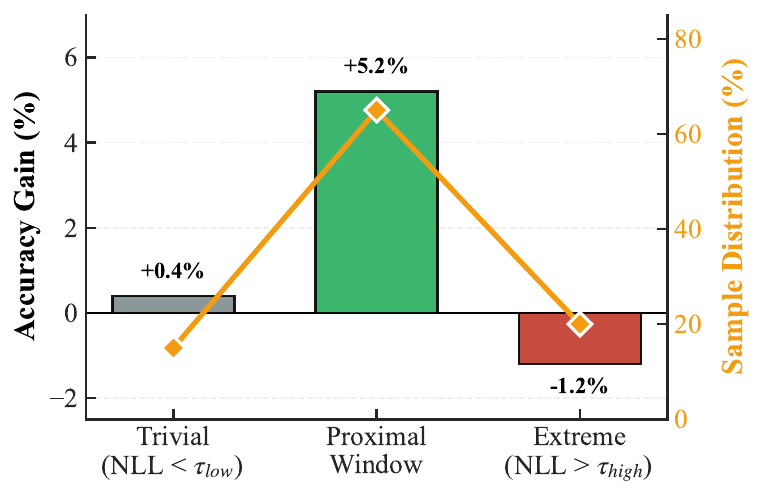}
    \includegraphics[width=0.305\linewidth]{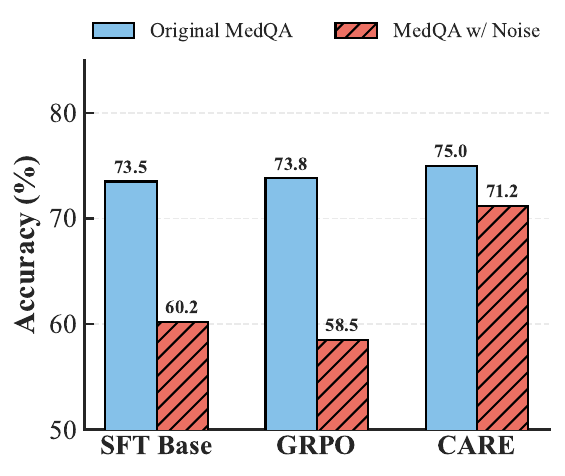}
    \caption{\small  \textbf{(Left)} Gradient $L_2$ Norm Tracking of CARE and GRPO. \textbf{(Middle)} Efficacy of the Proximal Window. \textbf{(Right)} Robustness to Spurious Clinical Distractors.}
    \label{fig:gradient_variance}
    \label{fig:learnability_regions}
    \label{fig:robustness_ood}
\end{figure}

\paragraph{Efficacy of the Proximal Development Zone.}
To further dissect \emph{why} the dynamic learnability window is crucial, we isolated the generated rollouts into three distinct NLL regions: Trivial (NLL $< \tau_{\text{low}}$), Proximal Window ($\tau_{\text{low}} \le$ NLL $\le \tau_{\text{high}}$), and Extreme Hard (NLL $> \tau_{\text{high}}$). We then conducted separate isolated training sessions utilizing only the experiences from each specific region. As shown in Figure \ref{fig:learnability_regions}(Middle), relying solely on \emph{Trivial} samples yields a negligible accuracy gain (+0.4\%), as these samples provide vanishing gradients and minimal information gain. Conversely, forcing the model to learn from \emph{Extreme Hard} samples results in a performance degradation (-1.2\%) due to the introduction of chaotic, noisy signals that corrupt learned weights. The vast majority of the performance improvement (+5.2\%) is derived entirely from the \emph{Proximal Window}. This analysis validates our dynamic filtering mechanism: efficient self-improvement requires isolating the effective learning experiences from statistical noise.

\paragraph{Causal Sufficiency vs. Spurious Shortcuts (Theorem \ref{thm:debiasing}).}
Theorem \ref{thm:debiasing} posits that our agreement-based self-verification acts as a $do$-calculus intervention to debias the gradient from spurious shortcuts.
To validate that agreement-based self-verification acts as a $do$-calculus intervention, we tested OOD robustness by injecting irrelevant symptoms and noise into MedQA queries. 
As shown in Figure \ref{fig:robustness_ood}(Right), the SFT Base model experiences a $13.3\%$ performance drop under perturbation. Strikingly, Standard GRPO suffers an even more severe drop of $15.3\%$. This perfectly demonstrates the ``Autoregressive Shortcut Trap'' (Theorem \ref{thm:failure}): unconstrained outcome-based RL inadvertently reinforces the model's reliance on superficial keyword matching, making it highly fragile.
In stark contrast, CARE exhibits remarkable robustness, degrading by only $3.8\%$. This confirms that enforcing decision-rationale consistency ($\phi_{\text{causal}}$) successfully severs the model's dependence on dataset artifacts, forcing it to extract the true causal chain required to reach the diagnosis.

\section{Conclusion}
In this work, we introduced Causally-Aligned Reasoning Exploration, a self-improvement framework designed to mitigate the \textit{causal misalignment} inherent in outcome-based medical reinforcement learning. By shifting the optimization paradigm from naive outcome monitoring to rigorous intrinsic experience curation, governed by causal sufficiency and proximal learnability, CARE ensures that policy updates are driven by valid clinical deduction rather than spurious dataset shortcuts. Extensive evaluations across diverse multimodal and text-only benchmarks demonstrate that CARE consistently outperforms strong SFT and GRPO baselines while significantly enhancing reasoning consistency.

\bibliography{reference}

@String(NeurIPS = {Adv. Neural Inform. Process. Syst.})

@inproceedings{zhang2023huatuogpt,
  author       = {Hongbo Zhang and
                  Junying Chen and
                  Feng Jiang and
                  Fei Yu and
                  Zhihong Chen and
                  Guiming Chen and
                  Jianquan Li and
                  Xiangbo Wu and
                  Zhiyi Zhang and
                  Qingying Xiao and
                  Xiang Wan and
                  Benyou Wang and
                  Haizhou Li},
  title        = {HuatuoGPT, Towards Taming Language Model to Be a Doctor},
  booktitle    = {Findings of the Association for Computational Linguistics: {EMNLP}
                  2023, Singapore, December 6-10, 2023},
  pages        = {10859--10885},
  publisher    = {Association for Computational Linguistics},
  year         = {2023},
  url          = {https://doi.org/10.18653/v1/2023.findings-emnlp.725},
  doi          = {10.18653/V1/2023.FINDINGS-EMNLP.725},
  bibsource    = {dblp computer science bibliography, https://dblp.org}
}

@article{yang2023zhongjing,
  author       = {Songhua Yang and
                  Hanjie Zhao and
                  Senbin Zhu and
                  Guangyu Zhou and
                  Hongfei Xu and
                  Yuxiang Jia and
                  Hongying Zan},
  title        = {Zhongjing: Enhancing the Chinese Medical Capabilities of Large Language
                  Model through Expert Feedback and Real-world Multi-turn Dialogue},
  journal      = {CoRR},
  volume       = {abs/2308.03549},
  year         = {2023},
  url          = {https://doi.org/10.48550/arXiv.2308.03549},
  doi          = {10.48550/ARXIV.2308.03549},
  eprinttype    = {arXiv},
  eprint       = {2308.03549},
  bibsource    = {dblp computer science bibliography, https://dblp.org}
}

@article{2024huatuogpt-o1,
  author       = {Junying Chen and
                  Zhenyang Cai and
                  Ke Ji and
                  Xidong Wang and
                  Wanlong Liu and
                  Rongsheng Wang and
                  Jianye Hou and
                  Benyou Wang},
  title        = {HuatuoGPT-o1, Towards Medical Complex Reasoning with LLMs},
  journal      = {CoRR},
  volume       = {abs/2412.18925},
  year         = {2024},
  url          = {https://doi.org/10.48550/arXiv.2412.18925},
  doi          = {10.48550/ARXIV.2412.18925},
  eprinttype    = {arXiv},
  eprint       = {2412.18925},
  bibsource    = {dblp computer science bibliography, https://dblp.org}
}

@article{2025medvlm-r1,
  author       = {Jiazhen Pan and
                  Che Liu and
                  Junde Wu and
                  Fenglin Liu and
                  Jiayuan Zhu and
                  Hongwei Bran Li and
                  Chen Chen and
                  Cheng Ouyang and
                  Daniel Rueckert},
  title        = {MedVLM-R1: Incentivizing Medical Reasoning Capability of Vision-Language
                  Models (VLMs) via Reinforcement Learning},
  journal      = {CoRR},
  volume       = {abs/2502.19634},
  year         = {2025},
  url          = {https://doi.org/10.48550/arXiv.2502.19634},
  doi          = {10.48550/ARXIV.2502.19634},
  eprinttype    = {arXiv},
  eprint       = {2502.19634},
  bibsource    = {dblp computer science bibliography, https://dblp.org}
}

@article{baichuan-m2,
  author       = {Chengfeng Dou and
                  Chong Liu and
                  Fan Yang and
                  Fei Li and
                  Jiyuan Jia and
                  Mingyang Chen and
                  Qiang Ju and
                  Shuai Wang and
                  Shunya Dang and
                  Tianpeng Li and
                  Xiangrong Zeng and
                  Yijie Zhou and
                  Chenzheng Zhu and
                  Da Pan and
                  Fei Deng and
                  Guangwei Ai and
                  Guosheng Dong and
                  Hongda Zhang and
                  Jinyang Tai and
                  Jixiang Hong and
                  Kai Lu and
                  Linzhuang Sun and
                  Peidong Guo and
                  Qian Ma and
                  Rihui Xin and
                  Shihui Yang and
                  Shusen Zhang and
                  Yichuan Mo and
                  Zheng Liang and
                  Zhishou Zhang and
                  Hengfu Cui and
                  Zuyi Zhu and
                  Xiaochuan Wang},
  title        = {Baichuan-M2: Scaling Medical Capability with Large Verifier System},
  journal      = {CoRR},
  volume       = {abs/2509.02208},
  year         = {2025},
  url          = {https://doi.org/10.48550/arXiv.2509.02208},
  doi          = {10.48550/ARXIV.2509.02208},
  eprinttype    = {arXiv},
  eprint       = {2509.02208},
  bibsource    = {dblp computer science bibliography, https://dblp.org}
}

@article{rarl-vlm,
  author       = {Tan{-}Hanh Pham and
                  Chris Ngo},
  title        = {{RARL:} Improving Medical {VLM} Reasoning and Generalization with
                  Reinforcement Learning and LoRA under Data and Hardware Constraints},
  journal      = {CoRR},
  volume       = {abs/2506.06600},
  year         = {2025},
  url          = {https://doi.org/10.48550/arXiv.2506.06600},
  doi          = {10.48550/ARXIV.2506.06600},
  eprinttype    = {arXiv},
  eprint       = {2506.06600},
  bibsource    = {dblp computer science bibliography, https://dblp.org}
}

@article{lingshu,
  author       = {LASA Team and
                  Weiwen Xu and
                  Hou Pong Chan and
                  Long Li and
                  Mahani Aljunied and
                  Ruifeng Yuan and
                  Jianyu Wang and
                  Chenghao Xiao and
                  Guizhen Chen and
                  Chaoqun Liu and
                  Zhaodonghui Li and
                  Yu Sun and
                  Junao Shen and
                  Chaojun Wang and
                  Jie Tan and
                  Deli Zhao and
                  Tingyang Xu and
                  Hao Zhang and
                  Yu Rong},
  title        = {Lingshu: {A} Generalist Foundation Model for Unified Multimodal Medical
                  Understanding and Reasoning},
  journal      = {CoRR},
  volume       = {abs/2506.07044},
  year         = {2025},
  url          = {https://doi.org/10.48550/arXiv.2506.07044},
  doi          = {10.48550/ARXIV.2506.07044},
  eprinttype    = {arXiv},
  eprint       = {2506.07044},
  bibsource    = {dblp computer science bibliography, https://dblp.org}
}

@article{hulu-med,
  author       = {Songtao Jiang and
                  Yuan Wang and
                  Sibo Song and
                  Tianxiang Hu and
                  Chenyi Zhou and
                  Bin Pu and
                  Yan Zhang and
                  Zhibo Yang and
                  Yang Feng and
                  Joey Tianyi Zhou and
                  Jin Hao and
                  Zijian Chen and
                  Ruijia Wu and
                  Tao Tang and
                  Junhui Lv and
                  Hongxia Xu and
                  Hongwei Wang and
                  Jun Xiao and
                  Bin Feng and
                  Fudong Zhu and
                  Kenli Li and
                  Weidi Xie and
                  Jimeng Sun and
                  Jian Wu and
                  Zuozhu Liu},
  title        = {Hulu-Med: {A} Transparent Generalist Model towards Holistic Medical
                  Vision-Language Understanding},
  journal      = {CoRR},
  volume       = {abs/2510.08668},
  year         = {2025},
  url          = {https://doi.org/10.48550/arXiv.2510.08668},
  doi          = {10.48550/ARXIV.2510.08668},
  eprinttype    = {arXiv},
  eprint       = {2510.08668},
  bibsource    = {dblp computer science bibliography, https://dblp.org}
}

@inproceedings{policy_gradient_methods,
  author       = {Richard S. Sutton and
                  David A. McAllester and
                  Satinder Singh and
                  Yishay Mansour},
  title        = {Policy Gradient Methods for Reinforcement Learning with Function Approximation},
  booktitle    = {Advances in Neural Information Processing Systems 12, {[NIPS} Conference,
                  Denver, Colorado, USA, November 29 - December 4, 1999]},
  pages        = {1057--1063},
  publisher    = {The {MIT} Press},
  year         = {1999},
  url          = {http://papers.nips.cc/paper/1713-policy-gradient-methods-for-reinforcement-learning-with-function-approximation},
  bibsource    = {dblp computer science bibliography, https://dblp.org}
}

@article{actor-critic,
  author       = {Thomas Degris and
                  Martha White and
                  Richard S. Sutton},
  title        = {Off-Policy Actor-Critic},
  journal      = {CoRR},
  volume       = {abs/1205.4839},
  year         = {2012},
  url          = {http://arxiv.org/abs/1205.4839},
  eprinttype    = {arXiv},
  eprint       = {1205.4839},
  bibsource    = {dblp computer science bibliography, https://dblp.org}
}

@article{ppo,
  author       = {John Schulman and
                  Filip Wolski and
                  Prafulla Dhariwal and
                  Alec Radford and
                  Oleg Klimov},
  title        = {Proximal Policy Optimization Algorithms},
  journal      = {CoRR},
  volume       = {abs/1707.06347},
  year         = {2017},
  url          = {http://arxiv.org/abs/1707.06347},
  eprinttype    = {arXiv},
  eprint       = {1707.06347},
  bibsource    = {dblp computer science bibliography, https://dblp.org}
}

@inproceedings{decision_transformer,
  author       = {Lili Chen and
                  Kevin Lu and
                  Aravind Rajeswaran and
                  Kimin Lee and
                  Aditya Grover and
                  Michael Laskin and
                  Pieter Abbeel and
                  Aravind Srinivas and
                  Igor Mordatch},
  title        = {Decision Transformer: Reinforcement Learning via Sequence Modeling},
  booktitle    = {Advances in Neural Information Processing Systems 34: Annual Conference
                  on Neural Information Processing Systems 2021, NeurIPS 2021, December
                  6-14, 2021, virtual},
  pages        = {15084--15097},
  year         = {2021},
  url          = {https://proceedings.neurips.cc/paper/2021/hash/7f489f642a0ddb10272b5c31057f0663-Abstract.html},
  bibsource    = {dblp computer science bibliography, https://dblp.org}
}

@inproceedings{RLHF,
  author       = {Long Ouyang and
                  Jeffrey Wu and
                  Xu Jiang and
                  Diogo Almeida and
                  Carroll L. Wainwright and
                  Pamela Mishkin and
                  Chong Zhang and
                  Sandhini Agarwal and
                  Katarina Slama and
                  Alex Ray and
                  John Schulman and
                  Jacob Hilton and
                  Fraser Kelton and
                  Luke Miller and
                  Maddie Simens and
                  Amanda Askell and
                  Peter Welinder and
                  Paul F. Christiano and
                  Jan Leike and
                  Ryan Lowe},
  title        = {Training language models to follow instructions with human feedback},
  booktitle    = {Advances in Neural Information Processing Systems 35: Annual Conference
                  on Neural Information Processing Systems 2022, NeurIPS 2022, New Orleans,
                  LA, USA, November 28 - December 9, 2022},
  year         = {2022},
  url          = {http://papers.nips.cc/paper\_files/paper/2022/hash/b1efde53be364a73914f58805a001731-Abstract-Conference.html},
  bibsource    = {dblp computer science bibliography, https://dblp.org}
}

@article{StarPO,
  author       = {Zihan Wang and
                  Kangrui Wang and
                  Qineng Wang and
                  Pingyue Zhang and
                  Linjie Li and
                  Zhengyuan Yang and
                  Xing Jin and
                  Kefan Yu and
                  Minh Nhat Nguyen and
                  Licheng Liu and
                  Eli Gottlieb and
                  Yiping Lu and
                  Kyunghyun Cho and
                  Jiajun Wu and
                  Li Fei{-}Fei and
                  Lijuan Wang and
                  Yejin Choi and
                  Manling Li},
  title        = {{RAGEN:} Understanding Self-Evolution in {LLM} Agents via Multi-Turn
                  Reinforcement Learning},
  journal      = {CoRR},
  volume       = {abs/2504.20073},
  year         = {2025},
  url          = {https://doi.org/10.48550/arXiv.2504.20073},
  doi          = {10.48550/ARXIV.2504.20073},
  eprinttype    = {arXiv},
  eprint       = {2504.20073},
  bibsource    = {dblp computer science bibliography, https://dblp.org}
}

@article{GRPO,
  author       = {DeepSeek{-}AI},
  title        = {DeepSeek-R1: Incentivizing Reasoning Capability in LLMs via Reinforcement
                  Learning},
  journal      = {CoRR},
  volume       = {abs/2501.12948},
  year         = {2025},
  url          = {https://doi.org/10.48550/arXiv.2501.12948},
  doi          = {10.48550/ARXIV.2501.12948},
  eprinttype    = {arXiv},
  eprint       = {2501.12948},
  bibsource    = {dblp computer science bibliography, https://dblp.org}
}

@article{exgrpo,
  author       = {Runzhe Zhan and
                  Yafu Li and
                  Zhi Wang and
                  Xiaoye Qu and
                  Dongrui Liu and
                  Jing Shao and
                  Derek F. Wong and
                  Yu Cheng},
  title        = {ExGRPO: Learning to Reason from Experience},
  journal      = {CoRR},
  volume       = {abs/2510.02245},
  year         = {2025},
  url          = {https://doi.org/10.48550/arXiv.2510.02245},
  doi          = {10.48550/ARXIV.2510.02245},
  eprinttype    = {arXiv},
  eprint       = {2510.02245},
  bibsource    = {dblp computer science bibliography, https://dblp.org}
}

@article{healthGPT,
  author       = {Tianwei Lin and
                  Wenqiao Zhang and
                  Sijing Li and
                  Yuqian Yuan and
                  Binhe Yu and
                  Haoyuan Li and
                  Wanggui He and
                  Hao Jiang and
                  Mengze Li and
                  Xiaohui Song and
                  Siliang Tang and
                  Jun Xiao and
                  Hui Lin and
                  Yueting Zhuang and
                  Beng Chin Ooi},
  title        = {HealthGPT: {A} Medical Large Vision-Language Model for Unifying Comprehension
                  and Generation via Heterogeneous Knowledge Adaptation},
  journal      = {CoRR},
  volume       = {abs/2502.09838},
  year         = {2025},
  url          = {https://doi.org/10.48550/arXiv.2502.09838},
  doi          = {10.48550/ARXIV.2502.09838},
  eprinttype    = {arXiv},
  eprint       = {2502.09838},
  bibsource    = {dblp computer science bibliography, https://dblp.org}
}

@article{openai,
  author       = {Aaron Jaech and
                  Adam Kalai and
                  Adam Lerer and
                  Adam Richardson and
                  Ahmed El{-}Kishky and
                  Aiden Low and
                  Alec Helyar and
                  Aleksander Madry and
                  Alex Beutel and
                  Alex Carney and
                  Alex Iftimie and
                  Alex Karpenko and
                  Alex Tachard Passos and
                  Alexander Neitz and
                  Alexander Prokofiev and
                  Alexander Wei and
                  Allison Tam and
                  Ally Bennett and
                  Ananya Kumar and
                  Andre Saraiva and
                  Andrea Vallone and
                  Andrew Duberstein and
                  Andrew Kondrich and
                  Andrey Mishchenko and
                  Andy Applebaum and
                  Angela Jiang and
                  Ashvin Nair and
                  Barret Zoph and
                  Behrooz Ghorbani and
                  Ben Rossen and
                  Benjamin Sokolowsky and
                  Boaz Barak and
                  Bob McGrew and
                  Borys Minaiev and
                  Botao Hao and
                  Bowen Baker and
                  Brandon Houghton and
                  Brandon McKinzie and
                  Brydon Eastman and
                  Camillo Lugaresi and
                  Cary Bassin and
                  Cary Hudson and
                  Chak Ming Li and
                  Charles de Bourcy and
                  Chelsea Voss and
                  Chen Shen and
                  Chong Zhang and
                  Chris Koch and
                  Chris Orsinger and
                  Christopher Hesse and
                  Claudia Fischer and
                  Clive Chan and
                  Dan Roberts and
                  Daniel Kappler and
                  Daniel Levy and
                  Daniel Selsam and
                  David Dohan and
                  David Farhi and
                  David Mely and
                  David Robinson and
                  Dimitris Tsipras and
                  Doug Li and
                  Dragos Oprica and
                  Eben Freeman and
                  Eddie Zhang and
                  Edmund Wong and
                  Elizabeth Proehl and
                  Enoch Cheung and
                  Eric Mitchell and
                  Eric Wallace and
                  Erik Ritter and
                  Evan Mays and
                  Fan Wang and
                  Felipe Petroski Such and
                  Filippo Raso and
                  Florencia Leoni and
                  Foivos Tsimpourlas and
                  Francis Song and
                  Fred von Lohmann and
                  Freddie Sulit and
                  Geoff Salmon and
                  Giambattista Parascandolo and
                  Gildas Chabot and
                  Grace Zhao and
                  Greg Brockman and
                  Guillaume Leclerc and
                  Hadi Salman and
                  Haiming Bao and
                  Hao Sheng and
                  Hart Andrin and
                  Hessam Bagherinezhad and
                  Hongyu Ren and
                  Hunter Lightman and
                  Hyung Won Chung and
                  Ian Kivlichan and
                  Ian O'Connell and
                  Ian Osband and
                  Ignasi Clavera Gilaberte and
                  Ilge Akkaya},
  title        = {OpenAI o1 System Card},
  journal      = {CoRR},
  volume       = {abs/2412.16720},
  year         = {2024},
  url          = {https://doi.org/10.48550/arXiv.2412.16720},
  doi          = {10.48550/ARXIV.2412.16720},
  eprinttype    = {arXiv},
  eprint       = {2412.16720},
  bibsource    = {dblp computer science bibliography, https://dblp.org}
}

@article{vision-r1,
  author       = {Wenxuan Huang and
                  Bohan Jia and
                  Zijie Zhai and
                  Shaosheng Cao and
                  Zheyu Ye and
                  Fei Zhao and
                  Zhe Xu and
                  Yao Hu and
                  Shaohui Lin},
  title        = {Vision-R1: Incentivizing Reasoning Capability in Multimodal Large
                  Language Models},
  journal      = {CoRR},
  volume       = {abs/2503.06749},
  year         = {2025},
  url          = {https://doi.org/10.48550/arXiv.2503.06749},
  doi          = {10.48550/ARXIV.2503.06749},
  eprinttype    = {arXiv},
  eprint       = {2503.06749},
  bibsource    = {dblp computer science bibliography, https://dblp.org}
}

@article{med-r1,
  author       = {Yuxiang Lai and
                  Jike Zhong and
                  Ming Li and
                  Shitian Zhao and
                  Xiaofeng Yang},
  title        = {Med-R1: Reinforcement Learning for Generalizable Medical Reasoning
                  in Vision-Language Models},
  journal      = {CoRR},
  volume       = {abs/2503.13939},
  year         = {2025},
  url          = {https://doi.org/10.48550/arXiv.2503.13939},
  doi          = {10.48550/ARXIV.2503.13939},
  eprinttype    = {arXiv},
  eprint       = {2503.13939},
  bibsource    = {dblp computer science bibliography, https://dblp.org}
}

@article{reasonmed,
  author       = {Yu Sun and
                  Xingyu Qian and
                  Weiwen Xu and
                  Hao Zhang and
                  Chenghao Xiao and
                  Long Li and
                  Yu Rong and
                  Wenbing Huang and
                  Qifeng Bai and
                  Tingyang Xu},
  title        = {ReasonMed: {A} 370K Multi-Agent Generated Dataset for Advancing Medical
                  Reasoning},
  journal      = {CoRR},
  volume       = {abs/2506.09513},
  year         = {2025},
  url          = {https://doi.org/10.48550/arXiv.2506.09513},
  doi          = {10.48550/ARXIV.2506.09513},
  eprinttype    = {arXiv},
  eprint       = {2506.09513},
  bibsource    = {dblp computer science bibliography, https://dblp.org}
}

@article{mmedlm,
  author       = {Pengcheng Qiu and
                  Chaoyi Wu and
                  Xiaoman Zhang and
                  Weixiong Lin and
                  Haicheng Wang and
                  Ya Zhang and
                  Yanfeng Wang and
                  Weidi Xie},
  title        = {Towards Building Multilingual Language Model for Medicine},
  journal      = {CoRR},
  volume       = {abs/2402.13963},
  year         = {2024},
  url          = {https://doi.org/10.48550/arXiv.2402.13963},
  doi          = {10.48550/ARXIV.2402.13963},
  eprinttype    = {arXiv},
  eprint       = {2402.13963},
  bibsource    = {dblp computer science bibliography, https://dblp.org}
}

@article{capo,
  author       = {Songtao Jiang and
                  Yuan Wang and
                  Ruizhe Chen and
                  Yan Zhang and
                  Ruilin Luo and
                  Bohan Lei and
                  Sibo Song and
                  Yang Feng and
                  Jimeng Sun and
                  Jian Wu and
                  Zuozhu Liu},
  title        = {{CAPO:} Reinforcing Consistent Reasoning in Medical Decision-Making},
  journal      = {CoRR},
  volume       = {abs/2506.12849},
  year         = {2025},
  url          = {https://doi.org/10.48550/arXiv.2506.12849},
  doi          = {10.48550/ARXIV.2506.12849},
  eprinttype    = {arXiv},
  eprint       = {2506.12849},
  bibsource    = {dblp computer science bibliography, https://dblp.org}
}

@article{metaGPT,
  author       = {Sirui Hong and
                  Xiawu Zheng and
                  Jonathan Chen and
                  Yuheng Cheng and
                  Jinlin Wang and
                  Ceyao Zhang and
                  Zili Wang and
                  Steven Ka Shing Yau and
                  Zijuan Lin and
                  Liyang Zhou and
                  Chenyu Ran and
                  Lingfeng Xiao and
                  Chenglin Wu},
  title        = {MetaGPT: Meta Programming for Multi-Agent Collaborative Framework},
  journal      = {CoRR},
  volume       = {abs/2308.00352},
  year         = {2023},
  url          = {https://doi.org/10.48550/arXiv.2308.00352},
  doi          = {10.48550/ARXIV.2308.00352},
  eprinttype    = {arXiv},
  eprint       = {2308.00352},
  bibsource    = {dblp computer science bibliography, https://dblp.org}
}

@inproceedings{chatdev,
  author       = {Chen Qian and
                  Wei Liu and
                  Hongzhang Liu and
                  Nuo Chen and
                  Yufan Dang and
                  Jiahao Li and
                  Cheng Yang and
                  Weize Chen and
                  Yusheng Su and
                  Xin Cong and
                  Juyuan Xu and
                  Dahai Li and
                  Zhiyuan Liu and
                  Maosong Sun},
  title        = {ChatDev: Communicative Agents for Software Development},
  booktitle    = {Proceedings of the 62nd Annual Meeting of the Association for Computational
                  Linguistics (Volume 1: Long Papers), {ACL} 2024, Bangkok, Thailand,
                  August 11-16, 2024},
  pages        = {15174--15186},
  publisher    = {Association for Computational Linguistics},
  year         = {2024},
  url          = {https://doi.org/10.18653/v1/2024.acl-long.810},
  doi          = {10.18653/V1/2024.ACL-LONG.810},
  bibsource    = {dblp computer science bibliography, https://dblp.org}
}

@article{memGPT,
  author       = {Charles Packer and
                  Vivian Fang and
                  Shishir G. Patil and
                  Kevin Lin and
                  Sarah Wooders and
                  Joseph E. Gonzalez},
  title        = {MemGPT: Towards LLMs as Operating Systems},
  journal      = {CoRR},
  volume       = {abs/2310.08560},
  year         = {2023},
  url          = {https://doi.org/10.48550/arXiv.2310.08560},
  doi          = {10.48550/ARXIV.2310.08560},
  eprinttype    = {arXiv},
  eprint       = {2310.08560},
  bibsource    = {dblp computer science bibliography, https://dblp.org}
}

@article{camel,
  author       = {Guohao Li and
                  Hasan Abed Al Kader Hammoud and
                  Hani Itani and
                  Dmitrii Khizbullin and
                  Bernard Ghanem},
  title        = {{CAMEL:} Communicative Agents for "Mind" Exploration of Large Scale
                  Language Model Society},
  journal      = {CoRR},
  volume       = {abs/2303.17760},
  year         = {2023},
  url          = {https://doi.org/10.48550/arXiv.2303.17760},
  doi          = {10.48550/ARXIV.2303.17760},
  eprinttype    = {arXiv},
  eprint       = {2303.17760},
  bibsource    = {dblp computer science bibliography, https://dblp.org}
}

@article{medcco,
  author       = {Shaohao Rui and
                  Kaitao Chen and
                  Weijie Ma and
                  Xiaosong Wang},
  title        = {Improving Medical Reasoning with Curriculum-Aware Reinforcement Learning},
  journal      = {CoRR},
  volume       = {abs/2505.19213},
  year         = {2025},
  url          = {https://doi.org/10.48550/arXiv.2505.19213},
  doi          = {10.48550/ARXIV.2505.19213},
  eprinttype    = {arXiv},
  eprint       = {2505.19213},
  bibsource    = {dblp computer science bibliography, https://dblp.org}
}

@article{refining,
  author       = {Muneerah Q. Alqahtani and
                  Abdullah Albarakati and
                  Fahd Alotaibi},
  title        = {Refining medical large language models: key insights from instruction
                  tuning},
  journal      = {PeerJ Comput. Sci.},
  volume       = {11},
  pages        = {e3216},
  year         = {2025},
  url          = {https://doi.org/10.7717/peerj-cs.3216},
  doi          = {10.7717/PEERJ-CS.3216},
  bibsource    = {dblp computer science bibliography, https://dblp.org}
}

@article{medprm,
  author       = {Jaehoon Yun and
                  Jiwoong Sohn and
                  Jungwoo Park and
                  Hyunjae Kim and
                  Xiangru Tang and
                  Yanjun Shao and
                  Yonghoe Koo and
                  Minhyeok Ko and
                  Qingyu Chen and
                  Mark Gerstein and
                  Michael Moor and
                  Jaewoo Kang},
  title        = {Med-PRM: Medical Reasoning Models with Stepwise, Guideline-verified
                  Process Rewards},
  journal      = {CoRR},
  volume       = {abs/2506.11474},
  year         = {2025},
  url          = {https://doi.org/10.48550/arXiv.2506.11474},
  doi          = {10.48550/ARXIV.2506.11474},
  eprinttype    = {arXiv},
  eprint       = {2506.11474},
  bibsource    = {dblp computer science bibliography, https://dblp.org}
}

@article{llm-powered,
  author       = {Henrik Voigt and
                  Yurina Sugamiya and
                  Kai Lawonn and
                  Sina Zarrie{\ss} and
                  Atsuo Takanishi},
  title        = {LLM-Powered Virtual Patient Agents for Interactive Clinical Skills
                  Training with Automated Feedback},
  journal      = {CoRR},
  volume       = {abs/2508.13943},
  year         = {2025},
  url          = {https://doi.org/10.48550/arXiv.2508.13943},
  doi          = {10.48550/ARXIV.2508.13943},
  eprinttype    = {arXiv},
  eprint       = {2508.13943},
  bibsource    = {dblp computer science bibliography, https://dblp.org}
}

@article{trpo,
  author       = {John Schulman and
                  Sergey Levine and
                  Philipp Moritz and
                  Michael I. Jordan and
                  Pieter Abbeel},
  title        = {Trust Region Policy Optimization},
  journal      = {CoRR},
  volume       = {abs/1502.05477},
  year         = {2015},
  url          = {http://arxiv.org/abs/1502.05477},
  eprinttype    = {arXiv},
  eprint       = {1502.05477},
  bibsource    = {dblp computer science bibliography, https://dblp.org}
}

@article{atari,
  author       = {Volodymyr Mnih and
                  Koray Kavukcuoglu and
                  David Silver and
                  Alex Graves and
                  Ioannis Antonoglou and
                  Daan Wierstra and
                  Martin A. Riedmiller},
  title        = {Playing Atari with Deep Reinforcement Learning},
  journal      = {CoRR},
  volume       = {abs/1312.5602},
  year         = {2013},
  url          = {http://arxiv.org/abs/1312.5602},
  eprinttype    = {arXiv},
  eprint       = {1312.5602},
  bibsource    = {dblp computer science bibliography, https://dblp.org}
}

@article{offline-rl,
  author       = {Sergey Levine and
                  Aviral Kumar and
                  George Tucker and
                  Justin Fu},
  title        = {Offline Reinforcement Learning: Tutorial, Review, and Perspectives
                  on Open Problems},
  journal      = {CoRR},
  volume       = {abs/2005.01643},
  year         = {2020},
  url          = {https://arxiv.org/abs/2005.01643},
  eprinttype    = {arXiv},
  eprint       = {2005.01643},
  bibsource    = {dblp computer science bibliography, https://dblp.org}
}

@article{autonomous-agent,
  author       = {Lei Wang and
                  Chen Ma and
                  Xueyang Feng and
                  Zeyu Zhang and
                  Hao Yang and
                  Jingsen Zhang and
                  Zhiyuan Chen and
                  Jiakai Tang and
                  Xu Chen and
                  Yankai Lin and
                  Wayne Xin Zhao and
                  Zhewei Wei and
                  Ji{-}Rong Wen},
  title        = {A Survey on Large Language Model based Autonomous Agents},
  journal      = {CoRR},
  volume       = {abs/2308.11432},
  year         = {2023},
  url          = {https://doi.org/10.48550/arXiv.2308.11432},
  doi          = {10.48550/ARXIV.2308.11432},
  eprinttype    = {arXiv},
  eprint       = {2308.11432},
  bibsource    = {dblp computer science bibliography, https://dblp.org}
}

@article{gmai-vl-r1,
  author       = {Yanzhou Su and
                  Tianbin Li and
                  Jiyao Liu and
                  Chenglong Ma and
                  Junzhi Ning and
                  Cheng Tang and
                  Sibo Ju and
                  Jin Ye and
                  Pengcheng Chen and
                  Ming Hu and
                  Shixiang Tang and
                  Lihao Liu and
                  Bin Fu and
                  Wenqi Shao and
                  Xiaowei Hu and
                  Xiangwen Liao and
                  Yuanfeng Ji and
                  Junjun He},
  title        = {{GMAI-VL-R1:} Harnessing Reinforcement Learning for Multimodal Medical
                  Reasoning},
  journal      = {CoRR},
  volume       = {abs/2504.01886},
  year         = {2025},
  url          = {https://doi.org/10.48550/arXiv.2504.01886},
  doi          = {10.48550/ARXIV.2504.01886},
  eprinttype    = {arXiv},
  eprint       = {2504.01886},
  bibsource    = {dblp computer science bibliography, https://dblp.org}
}

@article{lora,
  author       = {Edward J. Hu and
                  Yelong Shen and
                  Phillip Wallis and
                  Zeyuan Allen{-}Zhu and
                  Yuanzhi Li and
                  Shean Wang and
                  Weizhu Chen},
  title        = {LoRA: Low-Rank Adaptation of Large Language Models},
  journal      = {CoRR},
  volume       = {abs/2106.09685},
  year         = {2021},
  url          = {https://arxiv.org/abs/2106.09685},
  eprinttype    = {arXiv},
  eprint       = {2106.09685},
  bibsource    = {dblp computer science bibliography, https://dblp.org}
}

@article{advance-in-llm,
  author       = {Zhiyu Kan and
                  Wensheng Gan and
                  Zhenlian Qi and
                  Philip S. Yu},
  title        = {Advances in Large Language Models for Medicine},
  journal      = {CoRR},
  volume       = {abs/2509.18690},
  year         = {2025},
  url          = {https://doi.org/10.48550/arXiv.2509.18690},
  doi          = {10.48550/ARXIV.2509.18690},
  eprinttype    = {arXiv},
  eprint       = {2509.18690},
  bibsource    = {dblp computer science bibliography, https://dblp.org}
}

@article{ReST,
  author       = {{\c{C}}aglar G{\"{u}}l{\c{c}}ehre and
                  Tom Le Paine and
                  Srivatsan Srinivasan and
                  Ksenia Konyushkova and
                  Lotte Weerts and
                  Abhishek Sharma and
                  Aditya Siddhant and
                  Alex Ahern and
                  Miaosen Wang and
                  Chenjie Gu and
                  Wolfgang Macherey and
                  Arnaud Doucet and
                  Orhan Firat and
                  Nando de Freitas},
  title        = {Reinforced Self-Training (ReST) for Language Modeling},
  journal      = {CoRR},
  volume       = {abs/2308.08998},
  year         = {2023},
  url          = {https://doi.org/10.48550/arXiv.2308.08998},
  doi          = {10.48550/ARXIV.2308.08998},
  eprinttype    = {arXiv},
  eprint       = {2308.08998},
  bibsource    = {dblp computer science bibliography, https://dblp.org}
}

@article{deepseekmath,
  author       = {Zhihong Shao and
                  Peiyi Wang and
                  Qihao Zhu and
                  Runxin Xu and
                  Junxiao Song and
                  Mingchuan Zhang and
                  Y. K. Li and
                  Y. Wu and
                  Daya Guo},
  title        = {DeepSeekMath: Pushing the Limits of Mathematical Reasoning in Open
                  Language Models},
  journal      = {CoRR},
  volume       = {abs/2402.03300},
  year         = {2024},
  url          = {https://doi.org/10.48550/arXiv.2402.03300},
  doi          = {10.48550/ARXIV.2402.03300},
  eprinttype    = {arXiv},
  eprint       = {2402.03300},
  bibsource    = {dblp computer science bibliography, https://dblp.org}
}

@inproceedings{research-in-development,
  author       = {Chunfang Zhou and
                  Qingyue Gong and
                  Jinyang Zhu and
                  Huidan Luan},
  title        = {Research and Application of Large Language Models in HealthcareCurrent
                  Development of Large Language Models in the Healthcare FieldA Framework
                  for Applying Large Language Models and the Opportunities and Challenges
                  of Large Language Models in Healthcare: {A} Framework for Applying
                  Large Language Models and the Opportunities and Challenges of Large
                  Language Models in Healthcare},
  booktitle    = {Proceedings of the 2023 4th International Symposium on Artificial
                  Intelligence for Medicine Science, {ISAIMS} 2023, Chengdu, China,
                  October 20-22, 2023},
  pages        = {664--670},
  publisher    = {{ACM}},
  year         = {2023},
  url          = {https://doi.org/10.1145/3644116.3644226},
  doi          = {10.1145/3644116.3644226},
  bibsource    = {dblp computer science bibliography, https://dblp.org}
}

@article{rlfreason,
  author       = {Yiping Wang and
                  Qing Yang and
                  Zhiyuan Zeng and
                  Liliang Ren and
                  Lucas Liu and
                  Baolin Peng and
                  Hao Cheng and
                  Xuehai He and
                  Kuan Wang and
                  Jianfeng Gao and
                  Weizhu Chen and
                  Shuohang Wang and
                  Simon Shaolei Du and
                  Yelong Shen},
  title        = {Reinforcement Learning for Reasoning in Large Language Models with
                  One Training Example},
  journal      = {CoRR},
  volume       = {abs/2504.20571},
  year         = {2025},
  url          = {https://doi.org/10.48550/arXiv.2504.20571},
  doi          = {10.48550/ARXIV.2504.20571},
  eprinttype    = {arXiv},
  eprint       = {2504.20571},
  bibsource    = {dblp computer science bibliography, https://dblp.org}
}

@article{openai2025o3mini,
  title={Introducing OpenAI o3 and o4-mini},
  author={OpenAI, Team},
  journal={https://openai. com/index/introducing-o3-and-o4-mini/},
  year={2025}
}

@article{gpt4.1,
  title={Introducing GPT-4.1 in the API},
  author={OpenAI, Ananya Kumar and Yu, Jiahui and Hallman, John and others},
  journal={https://openai. com/index/gpt-4-1/},
  year={2025}
}

@article{gpt4o,
  author       = {Aaron Hurst and
                  Adam Lerer and
                  Adam P. Goucher and
                  Adam Perelman and
                  Aditya Ramesh and
                  Aidan Clark and
                  AJ Ostrow and
                  Akila Welihinda and
                  Alan Hayes and
                  Alec Radford and
                  Aleksander Madry and
                  Alex Baker{-}Whitcomb and
                  Alex Beutel and
                  Alex Borzunov and
                  Alex Carney and
                  Alex Chow and
                  Alex Kirillov and
                  Alex Nichol and
                  Alex Paino and
                  Alex Renzin and
                  Alex Tachard Passos and
                  Alexander Kirillov and
                  Alexi Christakis and
                  Alexis Conneau and
                  Ali Kamali and
                  Allan Jabri and
                  Allison Moyer and
                  Allison Tam and
                  Amadou Crookes and
                  Amin Tootoonchian and
                  Ananya Kumar and
                  Andrea Vallone and
                  Andrej Karpathy and
                  Andrew Braunstein and
                  Andrew Cann and
                  Andrew Codispoti and
                  Andrew Galu and
                  Andrew Kondrich and
                  Andrew Tulloch and
                  Andrey Mishchenko and
                  Angela Baek and
                  Angela Jiang and
                  Antoine Pelisse and
                  Antonia Woodford and
                  Anuj Gosalia and
                  Arka Dhar and
                  Ashley Pantuliano and
                  Avi Nayak and
                  Avital Oliver and
                  Barret Zoph and
                  Behrooz Ghorbani and
                  Ben Leimberger and
                  Ben Rossen and
                  Ben Sokolowsky and
                  Ben Wang and
                  Benjamin Zweig and
                  Beth Hoover and
                  Blake Samic and
                  Bob McGrew and
                  Bobby Spero and
                  Bogo Giertler and
                  Bowen Cheng and
                  Brad Lightcap and
                  Brandon Walkin and
                  Brendan Quinn and
                  Brian Guarraci and
                  Brian Hsu and
                  Bright Kellogg and
                  Brydon Eastman and
                  Camillo Lugaresi and
                  Carroll L. Wainwright and
                  Cary Bassin and
                  Cary Hudson and
                  Casey Chu and
                  Chad Nelson and
                  Chak Li and
                  Chan Jun Shern and
                  Channing Conger and
                  Charlotte Barette and
                  Chelsea Voss and
                  Chen Ding and
                  Cheng Lu and
                  Chong Zhang and
                  Chris Beaumont and
                  Chris Hallacy and
                  Chris Koch and
                  Christian Gibson and
                  Christina Kim and
                  Christine Choi and
                  Christine McLeavey and
                  Christopher Hesse and
                  Claudia Fischer and
                  Clemens Winter and
                  Coley Czarnecki and
                  Colin Jarvis and
                  Colin Wei and
                  Constantin Koumouzelis and
                  Dane Sherburn},
  title        = {GPT-4o System Card},
  journal      = {CoRR},
  volume       = {abs/2410.21276},
  year         = {2024},
  url          = {https://doi.org/10.48550/arXiv.2410.21276},
  doi          = {10.48550/ARXIV.2410.21276},
  eprinttype    = {arXiv},
  eprint       = {2410.21276},
  bibsource    = {dblp computer science bibliography, https://dblp.org}
}

@techreport{claude4,
  title={Introducing claude 4},
  author={Team, Anthropic},
  year={2025},
  institution={Technical report, Anthropic/Stanford CRFM}
}

@article{gemini2.5,
  author       = {Gemini Team},
  title        = {Gemini 2.5: Pushing the Frontier with Advanced Reasoning, Multimodality,
                  Long Context, and Next Generation Agentic Capabilities},
  journal      = {CoRR},
  volume       = {abs/2507.06261},
  year         = {2025},
  url          = {https://doi.org/10.48550/arXiv.2507.06261},
  doi          = {10.48550/ARXIV.2507.06261},
  eprinttype    = {arXiv},
  eprint       = {2507.06261},
  bibsource    = {dblp computer science bibliography, https://dblp.org}
}

@article{deepseekv3,
  author       = {DeepSeek{-}AI},
  title        = {DeepSeek-V3 Technical Report},
  journal      = {CoRR},
  volume       = {abs/2412.19437},
  year         = {2024},
  url          = {https://doi.org/10.48550/arXiv.2412.19437},
  doi          = {10.48550/ARXIV.2412.19437},
  eprinttype    = {arXiv},
  eprint       = {2412.19437},
  bibsource    = {dblp computer science bibliography, https://dblp.org}
}

@article{qwen2.5vl,
  author       = {Shuai Bai and
                  Keqin Chen and
                  Xuejing Liu and
                  Jialin Wang and
                  Wenbin Ge and
                  Sibo Song and
                  Kai Dang and
                  Peng Wang and
                  Shijie Wang and
                  Jun Tang and
                  Humen Zhong and
                  Yuanzhi Zhu and
                  Ming{-}Hsuan Yang and
                  Zhaohai Li and
                  Jianqiang Wan and
                  Pengfei Wang and
                  Wei Ding and
                  Zheren Fu and
                  Yiheng Xu and
                  Jiabo Ye and
                  Xi Zhang and
                  Tianbao Xie and
                  Zesen Cheng and
                  Hang Zhang and
                  Zhibo Yang and
                  Haiyang Xu and
                  Junyang Lin},
  title        = {Qwen2.5-VL Technical Report},
  journal      = {CoRR},
  volume       = {abs/2502.13923},
  year         = {2025},
  url          = {https://doi.org/10.48550/arXiv.2502.13923},
  doi          = {10.48550/ARXIV.2502.13923},
  eprinttype    = {arXiv},
  eprint       = {2502.13923},
  bibsource    = {dblp computer science bibliography, https://dblp.org}
}

@article{januspro,
  author       = {Xiaokang Chen and
                  Zhiyu Wu and
                  Xingchao Liu and
                  Zizheng Pan and
                  Wen Liu and
                  Zhenda Xie and
                  Xingkai Yu and
                  Chong Ruan},
  title        = {Janus-Pro: Unified Multimodal Understanding and Generation with Data
                  and Model Scaling},
  journal      = {CoRR},
  volume       = {abs/2501.17811},
  year         = {2025},
  url          = {https://doi.org/10.48550/arXiv.2501.17811},
  doi          = {10.48550/ARXIV.2501.17811},
  eprinttype    = {arXiv},
  eprint       = {2501.17811},
  bibsource    = {dblp computer science bibliography, https://dblp.org}
}

@article{intervl2.5,
  author       = {Zhe Chen and
                  Weiyun Wang and
                  Yue Cao and
                  Yangzhou Liu and
                  Zhangwei Gao and
                  Erfei Cui and
                  Jinguo Zhu and
                  Shenglong Ye and
                  Hao Tian and
                  Zhaoyang Liu and
                  Lixin Gu and
                  Xuehui Wang and
                  Qingyun Li and
                  Yimin Ren and
                  Zixuan Chen and
                  Jiapeng Luo and
                  Jiahao Wang and
                  Tan Jiang and
                  Bo Wang and
                  Conghui He and
                  Botian Shi and
                  Xingcheng Zhang and
                  Han Lv and
                  Yi Wang and
                  Wenqi Shao and
                  Pei Chu and
                  Zhongying Tu and
                  Tong He and
                  Zhiyong Wu and
                  Huipeng Deng and
                  Jiaye Ge and
                  Kai Chen and
                  Min Dou and
                  Lewei Lu and
                  Xizhou Zhu and
                  Tong Lu and
                  Dahua Lin and
                  Yu Qiao and
                  Jifeng Dai and
                  Wenhai Wang},
  title        = {Expanding Performance Boundaries of Open-Source Multimodal Models
                  with Model, Data, and Test-Time Scaling},
  journal      = {CoRR},
  volume       = {abs/2412.05271},
  year         = {2024},
  url          = {https://doi.org/10.48550/arXiv.2412.05271},
  doi          = {10.48550/ARXIV.2412.05271},
  eprinttype    = {arXiv},
  eprint       = {2412.05271},
  bibsource    = {dblp computer science bibliography, https://dblp.org}
}

@article{bimedix2,
  author       = {Sahal Shaji Mullappilly and
                  Mohammed Irfan Kurpath and
                  Sara Pieri and
                  Saeed Yahya Alseiari and
                  Shanavas Cholakkal and
                  Khaled Aldahmani and
                  Fahad Khan and
                  Rao Muhammad Anwer and
                  Salman Khan and
                  Timothy Baldwin and
                  Hisham Cholakkal},
  title        = {BiMediX2: Bio-Medical EXpert {LMM} for Diverse Medical Modalities},
  journal      = {CoRR},
  volume       = {abs/2412.07769},
  year         = {2024},
  url          = {https://doi.org/10.48550/arXiv.2412.07769},
  doi          = {10.48550/ARXIV.2412.07769},
  eprinttype    = {arXiv},
  eprint       = {2412.07769},
  bibsource    = {dblp computer science bibliography, https://dblp.org}
}

@article{PathVQA,
  author       = {Xuehai He and
                  Yichen Zhang and
                  Luntian Mou and
                  Eric P. Xing and
                  Pengtao Xie},
  title        = {PathVQA: 30000+ Questions for Medical Visual Question Answering},
  journal      = {CoRR},
  volume       = {abs/2003.10286},
  year         = {2020},
  url          = {https://arxiv.org/abs/2003.10286},
  eprinttype    = {arXiv},
  eprint       = {2003.10286},
  bibsource    = {dblp computer science bibliography, https://dblp.org}
}
\bibliographystyle{template_conference}
\clearpage
\appendix

\end{document}